\documentclass{article}
\ifdefined\pdfoutput\pdfoutput=1\fi
\usepackage{arxiv}
\usepackage{iftex}
\ifPDFTeX
  \usepackage[utf8]{inputenc}
  \usepackage[T1]{fontenc}
\else
  \usepackage{fontspec}
\fi
\usepackage{amsmath,amssymb,amsthm}
\usepackage{stmaryrd}
\usepackage{mathtools}
\usepackage{enumitem}
\usepackage{booktabs}
\usepackage{makecell}
\usepackage{hyperref}
\hypersetup{
	pdftitle={A framework for auditing grounding claims},
	pdfsubject={linguistics, mathematics, computer science, cognitive science, philosophy},
	pdfauthor={Daniel Quigley},
	pdfkeywords={symbol grounding, evaluation metrics, causal representation, compositionality,  robustness, semantics, large language models, embodied AI},
}
\usepackage[activate={true,nocompatibility},final]{microtype}
\usepackage{doi}
\usepackage{array}
\usepackage{cleveref}
\usepackage[most]{tcolorbox}
\usepackage{tikz}			                
\usepackage{tikz-qtree}			                
\usetikzlibrary{shapes.geometric}
\usetikzlibrary{arrows.meta,arrows}
\usetikzlibrary{3d,tikzmark,intersections,scopes,shapes,shapes.multipart,calc,positioning,fit,bending,backgrounds,decorations.markings,cd,shapes.geometric, arrows.meta, decorations.pathreplacing}
\definecolor{evalblue}{RGB}{66, 133, 244}
\definecolor{archgray}{RGB}{245, 245, 245}
\definecolor{archborder}{RGB}{180, 180, 180}
\definecolor{chaingreen}{RGB}{52, 168, 83}
\definecolor{targetorange}{RGB}{251, 188, 4}
\definecolor{desidpurple}{RGB}{154, 110, 188}
\definecolor{profilered}{RGB}{234, 67, 53}
\definecolor{typoteal}{RGB}{0, 150, 136}
\definecolor{modeblue}{RGB}{100, 149, 237}

\definecolor{boxcolor1}{RGB}{255,254,230}
\definecolor{boxcolor2}{RGB}{252,227,215}
\definecolor{boxcolor3}{RGB}{225,240,227}
\definecolor{boxcolor4}{RGB}{227,235,246}

\tcbset{
	envbox/.style={
		breakable,
		before skip=1.5\topskip,
		after skip=1.5\topskip,
		left skip=0pt,
		right skip=0pt,
		left=4pt,
		right=4pt,
		top=2pt,
		bottom=2pt,
		lefttitle=4pt,
		righttitle=4pt,
		toptitle=2pt,
		bottomtitle=2pt,
		sharp corners,
		boxrule=0pt,
		titlerule=.4pt,
		coltitle=black,
		colframe=darkgray,
		coltext=black,
		fonttitle=\bfseries,
	}
}
\newtcolorbox{definitionbox}[1][]{envbox, colback=boxcolor1, colbacktitle=boxcolor1, #1}
\newtcolorbox{theorembox}[1][]{envbox, colback=boxcolor2, colbacktitle=boxcolor2, #1}
\newtcolorbox{lemmabox}[1][]{envbox, colback=boxcolor3, colbacktitle=boxcolor3, #1}
\newtcolorbox{propositionbox}[1][]{envbox, colback=boxcolor3, colbacktitle=boxcolor3, #1}
\newtcolorbox{corollarybox}[1][]{envbox, colback=boxcolor4, colbacktitle=boxcolor4, #1}

\newtheoremstyle{mystyle}{3pt}{3pt}{\itshape}{}{\bfseries}{.}{.5em}{}
\theoremstyle{mystyle}

\newtheorem{remark}{Remark}[section]

\newtheorem{definition}{Definition}[section]
\renewenvironment{definition}[1][]{%
	\refstepcounter{definition}%
	\begin{definitionbox}[title=Definition~\thedefinition\ifx\relax#1\relax\else\ (#1)\fi]%
	}{%
	\end{definitionbox}
}
\newtheorem{theorem}{Theorem}[section]
\renewenvironment{theorem}[1][]{%
	\refstepcounter{theorem}%
	\begin{theorembox}[title=Theorem~\thetheorem\ifx\relax#1\relax\else\ (#1)\fi]%
	}{%
	\end{theorembox}
}
\newtheorem{lemma}{Lemma}[section]
\renewenvironment{lemma}[1][]{%
	\refstepcounter{lemma}%
	\begin{lemmabox}[title=Lemma~\thelemma\ifx\relax#1\relax\else\ (#1)\fi]%
	}{%
	\end{lemmabox}
}
\newtheorem{proposition}{Proposition}[section]
\renewenvironment{proposition}[1][]{%
	\refstepcounter{proposition}%
	\begin{propositionbox}[title=Proposition~\theproposition\ifx\relax#1\relax\else\ (#1)\fi]%
	}{%
	\end{propositionbox}
}
\newtheorem{corollary}{Corollary}[section]
\renewenvironment{corollary}[1][]{%
	\refstepcounter{corollary}%
	\begin{corollarybox}[title=Corollary~\thecorollary\ifx\relax#1\relax\else\ (#1)\fi]%
	}{%
	\end{corollarybox}
}

\newtheorem{observation}{Observation}[section]

\title{A framework for auditing grounding claims}

\author{
	\href{https://orcid.org/0009-0004-7957-1806}{\includegraphics[scale=0.06]{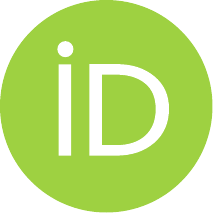}\hspace{1mm}Daniel Quigley} \\
	Center for Possible Minds\\
	Indiana University Bloomington\\
	Bloomington, IN 47408 \\
	\texttt{dgquigle@iu.edu} \\
	\And
	Eric Maynard \\
	Eruditis\\
	Milwaukee, WI 53207 \\
	\texttt{eric@eruditis.com} \\
}

\renewcommand{\headeright}{Preprint}
\renewcommand{\undertitle}{Preprint}
\renewcommand{\shorttitle}{Auditing grounding claims}

\begin{document}
\maketitle

\begin{abstract}\label{abs:abstract}
    The symbol grounding problem asks how a token such as \textit{cat} can be \emph{about} cats. We propose a framework for auditing grounding claims against a declared semantic standard. The audit reports measurements and evidence, with overall verdicts conditional on explicit acceptance criteria. Its profiles assess accuracy, robustness, and composition alongside evidence about how the system acquired its mechanisms, how they contribute to performance, and why they were retained. In a toy gridworld, an agent interprets individual symbols accurately but fails a withheld combination. Composing its interpretations by the declared rule would succeed. This comparison identifies a departure from the composition rule within the observed failure. Both this audit and a pilot on pretrained word vectors provide evidence that a designated mechanism contributes to present performance. Whether that contribution explains its retention remains uncertified. The framework evaluates the evidence for grounding claims; candidate accounts remain responsible for explaining how meaning emerges.
\end{abstract}


\section{Introduction}\label{sec:intro}

The \textbf{symbol grounding problem} asks how a token such as \textit{cat} can be \emph{about} cats. In Searle's Chinese Room \cite{searle1980minds}, a person follows rules for manipulating Chinese symbols as uninterpreted shapes; the thought experiment challenges the sufficiency of formal symbol manipulation for understanding. Harnad's formulation focuses on the regress of explaining symbols through other symbols and proposes connections to non-symbolic representations as a way to arrest it \cite{HARNAD1990335}. For an artificial agent, the problem arises when an instruction such as \emph{Pick up the red mug} must guide action toward an object in its environment.

Success on that instruction supplies evidence about grounding but leaves several questions open. The agent may recognize familiar combinations while failing on new combinations of familiar words. Its performance may depend on an incidental cue or on an interpretation supplied by an external observer. We propose an audit that distinguishes these possibilities. It assesses a system's symbols against a declared standard and reports a profile of semantic accuracy, causal and historical evidence, robustness, and composition, together with the provenance of the mechanisms involved.

In a grid-world agent, composing the agent's interpretations of familiar words would succeed on a withheld command, while its interpretation of the combined command fails. Both this audit and a pilot on pretrained word vectors provide evidence of a mechanism's present contribution to performance, while leaving its historical selection uncertified. Further, the audit distinguishes the system's mechanisms from the apparatus used to evaluate them. The auditor declares the context, the aspect of meaning at issue, the targets, and the test populations, perturbations, interventions, and tolerances. Authenticity records which mechanisms belong to the system and whether its own learning or evolution produced them. A readout supplied only for evaluation remains part of the instrument, and so does the interpretation it realizes. The resulting claims are relative to these declarations: a system may pass an audit of reference to objects, and fail audits of inference or social use.

We treat the proposal as an explication \cite{carnap1950logical}, assessed by the distinctions it preserves and the diagnoses it supports. The clauses address different questions about grounding claims. The zero semantical commitment condition requires an explanation of semantic emergence not to presuppose the semantic resources it seeks to explain \cite{Taddeo01122005}. The audit evaluates claims against a standard; candidate accounts explain how meaning emerges and remain answerable to that condition.

\section{Framework}\label{sec:genground}

We itemize with the imaginatively prefixed ``G'' for ``grounding'' to diagnostic clauses: authenticity, preservation, faithfulness, robustness, and compositionality, with faithfulness divided into correlational and etiological tiers, arranged in Table~\ref{tab:explanatory}. The definitions that follow specify the spaces, maps, populations, and tolerances used in those comparisons.

\begin{table}
	\centering
	\footnotesize
	\setlength{\tabcolsep}{4pt}
	\renewcommand{\arraystretch}{1.2}
	\caption{Diagnostic clauses}
	\label{tab:explanatory}
	\begin{tabular}{>{\raggedright\arraybackslash}p{2.6cm} >{\raggedright\arraybackslash}p{6.1cm} >{\raggedright\arraybackslash}p{6.5cm}}
		\toprule
		\textbf{Clause} & \textbf{Question} & \textbf{Comparison} \\
		\midrule
		\textbf{G0 authenticity} & Do the mechanisms that carry symbols to meanings belong to the system, and did its own learning or evolution produce them? & inspection of provenance for each component relied upon, relative to the declared boundary and history: membership, then acquisition \\
		\textbf{G1 preservation} & Does each atom mean what the standard says it means? & the realized meaning of each atom against its target, aggregated over the atoms \\
		\textbf{G2a correlational faithfulness} & Does each composite mean what the standard says it means? & the realized meaning of each composite against its target, aggregated over the composites \\
		\textbf{G2b etiological faithfulness} & Does the designated mechanism contribute to success now, and was it retained for that contribution? & success with the mechanism on against success with it off, on the decoupled distribution; historical controls for the retention \\
		\textbf{G3 robustness} & How far does realized meaning move under a declared perturbation? & the realized meaning of an item against the realized meaning of its perturbed self, as a curve over scales \\
		\textbf{G4 compositionality} & Does a composite mean what the algebra makes of the system's own parts, and do withheld combinations pass? & the realized composite against the algebra applied to the realized atoms; the pass rate over withheld composites \\
		\bottomrule
	\end{tabular}
\end{table}

A profile at an evaluation index under an audit specification consists of provenance records per component relied upon and measurement, the goal of this whole enterprise.

\begin{definition}\label{def:grounding}
	\textbf{Grounding}, relative to an evaluation index $E=(k,t,U,P)$ and an audit specification, is assessed by the degree to which an agent's own mechanisms implement mappings from its symbols into the meaning space $\mathcal M$ declared by $E$; satisfy the applicable authenticity tier, with G0-strong additionally requiring that the relevant mappings were acquired through the agent's own learning or evolution; realize meanings that align with the normative standard $\mathcal I$ supplied at the auditor level; include mechanisms selected for their contribution to success in the declared task family; remain stable under the threats in $U$; compose systematically over the grammar.
\end{definition}

A grounding profile may, then, be assigned to any auditable architecture, including one that satisfies G0-weak and fails G0-strong; the \emph{strong grounding claim} requires G0-strong, together with satisfaction of the remaining desiderata. Definition~\ref{def:grounding} quantifies a relation between an agent and a declared standard, with $\mathcal M$ and $\mathcal I$ as roles in the audit occupied differently by meaning type: for $t=\text{ext}$, the standard is derived from world-involving success conditions; for $t=\text{inf}$, from inferential practice; for $t=\text{soc}$, from normative social practice. Whether meanings supervene on internal states or include wide-world factors, the question \cite{Putnam1975PUTTMO} raised with Twin Earth and \cite{Burge1979BURIAT} extended to the social environment, stays open at the metaphysical level, since the audit requires only a declared standard. Pragmatists and constructivists, who reject pre-existing meanings, can occupy the standard role, since a standard can be read off the practice such positions take to be constitutive.

\subsection{Spaces and maps}\label{subsec:spaces}

The audit compares the meaning a system realizes with a declared target. Its basic maps are
\[
\Sigma \xrightarrow{ \Phi } \mathcal R
       \xrightarrow{ \mathcal A } \mathcal M,
\qquad
g:=\mathcal A\circ\Phi,
\qquad
\mathcal I:\Sigma\rightharpoonup\mathcal M.
\]
Here, $\Sigma$ contains the surface forms, $\mathcal R$ the internal representations, and $\mathcal M$ the meanings used for evaluation. The encoder $\Phi$ and readout $\mathcal A$ determine the realized interpretation $g$; the intended interpretation $\mathcal I$ supplies its targets.

An audit fixes a context $k\in K$ and a meaning type $t\in T=\{\text{ext},\text{inf},\text{soc}\}$: \emph{extensional}, \emph{inferential}, or \emph{social-normative}. Its meaning space is, therefore, $\mathcal M_k^t$, equipped with a (pseudo)metric $d_{k,t}$ that measures discrepancies between meanings. We suppress these indices when the context and type are fixed. The encoder is shared across contexts and types; the readout and, where applicable, the semantic algebra depend on them.

The surface forms in $\Sigma$ are whatever the agent manipulates as symbols, including text tokens, phonemes, program tokens, or action symbols. We distinguish atoms $\Sigma_{\text{atom}}\subseteq\Sigma$ from composites built by a grammar $\mathcal G$, whose typed constructors accept only well-typed tuples. The encoder maps these forms to internal states, as an embedding layer maps tokens to vectors. A (pseudo)metric $d_{\mathcal R}$ measures change in those states. Robustness tests perturb either the representations, through maps $u:\mathcal R\to\mathcal R$, or the surface forms, with their effects carried through $\Phi$. The perturbed space may be discrete or continuous.

The target interpretation $\mathcal I$ is specified on the atoms and composites included in the audit; it may be partial to accommodate missing entries, unresolved ambiguity, or items outside scope. When composition is evaluated, an algebra $\langle\mathcal M_k^t,\mathcal F_k^t\rangle$ interprets the grammar's constructors. For an unambiguous grammar, applying these operations along each parse extends atomic targets to composite targets; Appendix~\ref{sec:thms} gives the unique homomorphic extension. Preservation and correlational faithfulness compare $g$ with $\mathcal I$. Compositionality compares a realized composite with the algebra applied to its realized parts.

Under G0-weak, the encoder and any readout claimed as the agent's own must be implemented within the system; G0-strong additionally requires their acquisition through its own learning or evolution. An auditor-supplied readout is a measurement adapter and belongs to the instrument. In that case, G0 is assessed on $\Phi$ and the readout is recorded as stipulated. Evidence that the agent implements the adapter's interpretation has to come from the agent's own mechanisms. Etiological faithfulness separately asks whether the designated mechanism contributes to success and was retained for that contribution.

The choice of meaning space and metric depends on the task. For extensional content in vision and manipulation, meanings may be scene-indexed objects and relations, with discrepancy measured by $1-\operatorname{IoU}$ and a penalty for relational errors \cite{shridhar2023perceiver,goyal2023rvt}. Inferential content may use logical or relational structures with entailment-based or graph-edit measures \cite{Schubert2010EntailmentII,leeetal2025entailment,shietal2025natural}. For social-normative content, meanings may be dialogue commitments or normative statuses, with discrepancy tied to expected repair cost; measurement here is less settled \cite{dingemanse2024interactive}.

Throughout, a bare measurement such as $\varepsilon_{\text{pres}}$ denotes a population quantity, a hat its estimate, and a bar its declared bound. An argument, as in $\varepsilon(\sigma)$, denotes an item-level value. Profiles report estimates of the declared aggregates against their bounds.

\subsection{Levels of abstraction}\label{subsec:twolevels}

We adopt the method of levels of abstraction \cite{floridi2008method}. A \textbf{level of abstraction} is a finite set of observables, each consisting of a typed variable together with an interpretation of what it represents. Once those observables are declared, analysis is relative to that level. The observables of Section~\ref{subsec:spaces} separate into an agent level and an auditor level, a distinction between mechanisms being evaluated and the apparatus by which they are evaluated.

The agent level holds the encoder $\Phi$, the family of readouts $\mathcal A_{k}^{t}$ with their realized meanings $g_{k}^{t}=\mathcal A_{k}^{t}\circ\Phi$, the mechanisms $M$ that are internal components of these maps, and the process $\mathcal T$ that produced them.

At the auditor level, an evaluation declares an index $E=(k,t,U,P)$, with threat model $U$ and reference distribution $P$, and an audit specification: the target $\mathcal I$ and the metric $d$; the item populations; where the perturbations act, at what scales, and at what coverage; the designated mechanism $M$ with its intervention and the decoupling that derives $P^{\text{do}}$ from $P$; the bounds; the aggregations. The index locates the audit and the specification says how it is executed; the profile is a function of both. 

\begin{definition}[Agent level]\label{def:agentlevel}
	The \emph{agent level} $L_{\text{ag}}$ comprises the observables realized by the system under audit: the surface alphabet $\Sigma$ and, where the system realizes one, its grammar $\mathcal G$ (a grammar the auditor imposes on a system without one is instrument); the encoding $\Phi:\Sigma\to\mathcal R$; the readouts $\mathcal A_{k}^{t}:\mathcal R\to\mathcal M_{k}^{t}$; the internal mechanisms $M$; the acquisition history $\mathcal T$.
\end{definition}

\begin{definition}[Auditor level]\label{def:auditorlevel}
	The \emph{auditor level} $L_{\text{aud}}$ comprises the observables of the evaluation apparatus, the index and the specification: the index $E=(k,t,U,P)$ with its threat model $U$ and reference distribution $P$; the intended interpretation $\mathcal I$; the metrics $d$ and $d_{\mathcal R}$; the semantic algebra $\mathcal F$ where composition is audited; the item populations; the perturbation location, scales, and coverage level $\alpha$; the designation of the mechanism $M$, its intervention, and the decoupling that derives $P^{\text{do}}$ from $P$; the bounds $\bar\varepsilon_{\text{pres}}$, $\bar\varepsilon_{\text{faith}}$, $\bar\delta_{\text{comp}}$, $\bar\beta$, $\eta$, $\bar\omega$, and $\tau$; the aggregations.
\end{definition}

The meaning space $\mathcal M$ is shared: it is the codomain of the agent's readout $\mathcal A$, and it is the space over which the auditor declares the target $\mathcal I$, the metric $d$, and, where composition is audited, the algebra $\mathcal F$. The evaluation index selects the interface by fixing $(k,t)$, and, thereby, selects which member of the agent's family of readouts is audited. That member must be realized, and under strong authenticity acquired, on the agent's side. Sharing the space is the precondition of comparison, since two maps can be compared only where their values are commensurable.

We impose, and hold to for the remainder of the paper, that the specification of every observable of $L_{\text{ag}}$ is free of observables of $L_{\text{aud}}$. The interface $\mathcal M$ stands between the levels and appears on both sides as the type of the readout's codomain. The agent's mappings are defined, implemented, and acquired independently of the auditor's standard; the standard enters only when a profile is computed. Definitions~\ref{def:agentlevel} and \ref{def:auditorlevel} assign every symbol its level, with the interface marked as such.

\begin{proposition}[Auditor invariance]\label{prop:auditorinv}
	Let two auditors declare indices $E=(k,t,U,P)$ and $E'=(k,t,U',P')$ and specifications with distinct standards $(\mathcal I, d)$ and $(\mathcal I', d')$ over the same agent and the same interface $\mathcal M$. Then all observables of $L_{\text{ag}}$ are identical under both audits; only the grounding profile varies.
\end{proposition}

\begin{proof}
	Immediate from the discipline above: $\Phi$, $\mathcal A$, $M$, and $\mathcal T$ are specified in terms that exclude $L_{\text{aud}}$ observables, which enter only as parameters of the profile functional; the interface is common to both audits by hypothesis, so the readout $\mathcal A$ is one and the same map under both.
\end{proof}

Proposition~\ref{prop:auditorinv} states the sense in which the agent's semantics is invariant under the audit: the standard fixes a target of comparison, and the mappings compared against it are unchanged by the choice, as the property a thermometer measures is unchanged by the conventional fixed points of its scale \cite{chang2004inventing}. Auditors who disagree about the interface itself declare distinct meaning spaces at the same $(k,t)$ and conduct distinct audits, each requiring an agent-side readout into its declared space. Where the readout is supplied at audit time, it belongs to the standard, with the same conclusion.

\begin{definition}\label{def:arch}
	A \textbf{grounding architecture} is
	\[
	\mathfrak{G}=\left\langle
	\Sigma,\ \mathcal G,\ \mathcal R,\
	\{\mathcal M_{k}^{t}\}_{k,t},\
	\Phi,\ \{\mathcal A_{k}^{t}\}_{k,t},\ M,\ \mathcal T
	\right\rangle,
	\]
	with the realized meanings $g_{k}^{t}=\mathcal A_{k}^{t}\circ\Phi$ derived.
\end{definition}

 The components of Definition~\ref{def:arch} are the agent-level observables of Definition~\ref{def:agentlevel}, with the meaning spaces as interface; the intended interpretations $\mathcal I_{k}^{t}$ belong to the auditor level of Definition~\ref{def:auditorlevel}, and stand outside the architecture, as shown in Figure~\ref{fig:plaintext}.

\begin{figure}
	\centering
		\begin{tikzpicture}[
		>=Stealth,
		node distance=0.8cm and 1.2cm,
		evalbox/.style={
			rectangle, rounded corners=4pt,
			draw=evalblue, fill=evalblue!15,
			minimum width=6.75cm, minimum height=1.5cm,
			font=\small\bfseries, text=evalblue!80!black
		},
		chainnode/.style={
			rectangle, rounded corners=3pt,
			draw=chaingreen!80!black, fill=chaingreen!20,
			minimum width=1.3cm, minimum height=0.9cm,
			font=\small
		},
		targetnode/.style={
			rectangle, rounded corners=3pt,
			draw=targetorange!80!black, fill=targetorange!20,
			minimum width=1.8cm, minimum height=0.9cm,
			font=\small
		},
		desidbox/.style={
			rectangle, rounded corners=3pt,
			draw=desidpurple!80!black, fill=desidpurple!15,
			minimum width=1.4cm, minimum height=0.7cm,
			font=\footnotesize
		},
		profilebox/.style={
			rectangle, rounded corners=4pt,
			draw=profilered, fill=profilered!12,
			minimum width=7.50cm, minimum height=1.5cm,
			font=\small
		},
		modebox/.style={
			rectangle, rounded corners=3pt,
			draw=modeblue!80!black, fill=modeblue!15,
			minimum width=2.2cm, minimum height=0.6cm,
			font=\footnotesize\itshape
		},
		mapstyle/.style={
			->, thick, chaingreen!70!black
		},
		flowstyle/.style={
			->, thick, gray!70
		},
		auditarrow/.style={
			->, thick, desidpurple!70!black, dashed
		}
		]

		\node[evalbox] (eval) {};
		\node[anchor=north west, font=\small\bfseries, text=evalblue!80!black]
		at ($(eval.north west) + (0.15, -0.1)$) {Evaluation index $E$};
		\node[anchor=north east, font=\tiny, text=evalblue!70!black] at ($(eval.north east) + (-0.15, -0.1)$) {$L_{\text{aud}}$};
		\node[font=\small, text=evalblue!70!black] at ($(eval.center) + (0, -0.15)$) {
			(context $k$, meaning type $t$, threat $U$, distribution $P$)
		};

		\node[below=4.0cm of eval] (archcenter) {};
		\coordinate (archleft) at ($(archcenter) + (-4.5, 0)$);
		\coordinate (archright) at ($(archcenter) + (4.5, 0)$);

		\node[chainnode, align=center] (sigma) at ($(archcenter) + (-2.2, 0)$) {symbol\\[1pt]$\Sigma$};
		\node[chainnode, right=0.9cm of sigma, align=center] (R) {representation\\[1pt]$\mathcal R$};
		\node[chainnode, right=0.9cm of R, minimum width=1.6cm, align=center, draw=targetorange!80!black, densely dashed, fill=chaingreen!10] (M) {meaning\\[1pt]$\mathcal M_{k}^{t}$};
	\node[below=0.05cm of M, font=\tiny, text=gray] {interface};

		\draw[mapstyle] (sigma) -- node[above, font=\scriptsize] {$\Phi$} (R);
		\draw[mapstyle] (R) -- node[above, font=\scriptsize] {$\mathcal A_{k}^{t}$} (M);

		\draw[decorate, decoration={brace, amplitude=5pt, raise=2pt}, thick, gray!70]
		(sigma.north west) -- (M.north east)
		node[midway, above=8pt, font=\scriptsize, text=gray] {realized meaning $g_{k}^{t}=\mathcal A_{k}^{t}\circ\Phi$};

		\node[targetnode, above=2.25cm of M, align=center] (I) {intended interpretation\\[1pt]$\mathcal I_{k}^{t}$};
		\node[left=0.1cm of I, font=\tiny, text=targetorange!70!black] {(target, $L_{\text{aud}}$)};

		\draw[->, thick, targetorange!70!black, dashed] (I) -- (M);

		\node[modebox, left=0.75cm of sigma, yshift=1.0cm] (mode1) {symbolic};
		\node[modebox, below=0.15cm of mode1] (mode2) {vectorial};
		\node[modebox, below=0.15cm of mode2] (mode3) {relational};
		\node[modebox, below=0.15cm of mode3, minimum width=2.6cm] (mode4) {referential coupling};

		\node[
		fit=(mode1)(mode2)(mode3)(mode4),
		draw=modeblue!50!black, dashed, thick,
		rounded corners=4pt,
		inner sep=8pt,
		inner ysep=12pt
		] (modebox) {};

		\node[anchor=north west, font=\tiny\bfseries, text=modeblue!70!black] at ($(modebox.north west) + (0.1, -0.05)$) {Format and coupling};

		\begin{scope}[on background layer]
			\node[
			fit=(sigma)(M)(modebox),
			draw=archborder, fill=archgray,
			rounded corners=6pt,
			inner sep=15pt
			] (archbox) {};
		\end{scope}

		\node[anchor=north west, font=\small\bfseries] at ($(archbox.north west) + (0.15, -0.1)$) {Grounding architecture $\mathfrak{G}$};
		\node[anchor=north east, font=\tiny, text=gray] at ($(archbox.north east) + (-0.15, -0.1)$) {$L_{\text{ag}}$};

		\node[desidbox, right=1.2cm of archbox.east, yshift=1.5cm] (G0) {G0};
		\node[desidbox, below=0.2cm of G0] (G1) {G1};
		\node[desidbox, below=0.2cm of G1] (G2) {G2};
		\node[desidbox, below=0.2cm of G2] (G3) {G3};
		\node[desidbox, below=0.2cm of G3] (G4) {G4};
		\node[right=0.15cm of G0, font=\tiny, text=desidpurple!70!black] {authenticity};
		\node[right=0.15cm of G1, font=\tiny, text=desidpurple!70!black] {preservation};
		\node[right=0.15cm of G2, font=\tiny, text=desidpurple!70!black] {faithfulness};
		\node[right=0.15cm of G3, font=\tiny, text=desidpurple!70!black] {robustness};
		\node[right=0.15cm of G4, font=\tiny, text=desidpurple!70!black] {compositionality};

		\node[above=0.2cm of G0, font=\small\bfseries, text=desidpurple!80!black] {Desiderata};

		\draw[auditarrow] (archbox.east) -- (G0.west);
		\draw[auditarrow] (archbox.east) -- (G1.west);
		\draw[auditarrow] (archbox.east) -- (G2.west);
		\draw[auditarrow] (archbox.east) -- (G3.west);
		\draw[auditarrow] (archbox.east) -- (G4.west);

		\node[profilebox, minimum height=1.9cm, below=4.0cm of archcenter] (profile) {};
		\node[anchor=north west, font=\small\bfseries, text=profilered!80!black]
		at ($(profile.north west) + (0.15, -0.1)$) {Grounding profile};
		\node[anchor=north east, font=\tiny, text=profilered!70!black] at ($(profile.north east) + (-0.15, -0.1)$) {$L_{\text{aud}}$};
		\node[font=\footnotesize, text=profilered!70!black, align=center] at ($(profile.center) + (0, -0.3)$) {
			$\big\{$provenance records; certification status;\\[2pt]
			preservation error, faithfulness error, causal effect,\\[2pt]
			robustness modulus, composition error, systematicity
			$\big\}$
		};

		\coordinate (desidmid) at ($(G4) + (0, -0.5) $);
		\draw[flowstyle] (desidmid) |- node[pos=0.75, above, font=\tiny, text=gray] {yields} ($(profile.east) + (0, 0.3)$);

		\draw[flowstyle] (eval) -- node[left, font=\tiny, text=gray, xshift=-4pt] {parametrizes} (archbox.north);
		\draw[flowstyle] (archbox.south) -- node[right, font=\tiny, text=gray, xshift=2pt] {reports} (profile.north);

	\end{tikzpicture}
	\caption{The grounding architecture encodes surface symbols into internal representations and reads them out into a meaning space; its maps belong to the agent level $L_{\text{ag}}$. The evaluation index, the intended interpretation, and the profile belong to the auditor level $L_{\text{aud}}$; the meaning space is the interface shared by both, drawn at the boundary of the architecture.}
	\label{fig:plaintext}
\end{figure}

\subsection{Audit scope}\label{subsec:zobjection}\label{sec:zcondition}

The zero semantical commitment condition Z constrains
explanations of how symbols acquire meaning: an explanation
must not presuppose the semantic resources whose emergence
it seeks to explain \cite{Taddeo01122005}.
The distinction between agent and auditor gives the
declared standard a different role.
It supplies a basis for evaluating a grounding claim.
Explaining how the system acquired meaning is the
task of the candidate account.
A system may therefore pass the audit while the candidate
account explaining it violates Z.

Authenticity, G0, imposes a narrower requirement.
Mechanisms credited to the system must belong to its
declared boundary; the strong tier additionally requires
their acquisition through its own learning or evolution.
An auditor-supplied interpretation thereby counts as
instrument.
Original intentionality, and independence from an
interpreter's semantic resources, demand more than these
provenance requirements.

A separate objection concerns the standard's justification.
Agreement with $\mathcal I$ may show conformity to an
analyst's assignment.
Whether the assignment captures the agent's content is a
further question.
Making the assignment explicit leaves this objection
standing. It makes the dependence inspectable and permits
comparison with alternative standards.

The profile's coordinates have different dependencies.
Provenance concerns the mechanisms and their acquisition;
etiology concerns the declared success predicate.
Robustness depends on the metric and threats, while
compositional deviation also depends on the semantic
algebra. Atomic accuracy, composite accuracy, and
withheld-composite success compare realized meanings
directly with the target interpretation.
Evidence from these comparisons can support a grounding
claim, but its semantic interpretation remains conditional
on the justification of the standard.

\subsection{Desiderata}\label{subsec:desiderata}

We fix an index $E=(k,t,U,P)$, with $U$ a family of perturbations and $P$ a reference distribution over the instances the threat model addresses, representations or surface forms, together with an audit specification, and state each clause as a tolerance-indexed condition on agent-level observables. The index locates the audit. The specification supplies what each clause consumes: the target and metric for G1 and G2a, the item populations for every accuracy coordinate, the perturbation location, scale, and coverage for G3, the designated mechanism with its intervention and decoupling for G2b, and the bounds and aggregations throughout. Two audits at one index under different specifications may return different profiles, so a reported profile cites both.

\paragraph{G0 (authenticity)}
Authenticity records where the audited mechanisms operate and how they were acquired. The audit declares the system boundary and history $\mathcal T$, including whether pretraining belongs to that history. Each component of $\Phi$ and $\mathcal A$ relied upon in the evaluation receives a membership record and, if it belongs to the system, an acquisition record.

The membership record distinguishes system components, which operate within the declared boundary to produce the audited behavior, from instrument components supplied by the auditor. The latter include a measurement readout or a composition rule fitted to a system that has none.

For system components, the acquisition record distinguishes:
\begin{itemize}
    \item \emph{learned or evolved} as acquired through the system's own learning or evolution within $\mathcal T$;
    \item \emph{imported} as produced outside $\mathcal T$ and copied into the system, as with weights from another model;
    \item \emph{designed} as written by the system's builder.
\end{itemize}

Learning from external examples, rewards, or teacher outputs counts as own learning. The record identifies these sources, but the process that fixed the component's parameters determines its acquisition status.

G0-weak requires every component credited to the system to be a system member operative in the audited behavior. G0-strong additionally requires each such component to have been learned or evolved within $\mathcal T$. Imported or designed components, therefore, limit the verdict to G0-weak.

Instrument components are recorded separately and excluded from the system's provenance verdict. Results obtained through them remain conditional on the stipulated interpretation or composition rule. Presenting an audit-time instrument component as the system's own fails both tiers.

In this way, G0 concerns provenance. A learned mechanism may satisfy G0-strong regardless of why training retained it. Whether it was retained because of its contribution to the audited success property is the separate question posed by G2b.

\paragraph{G1 (preservation)}
Preservation measures atomic accuracy against the declared standard. For each item $\sigma$ in the domain of $\mathcal I$, define
\[
\varepsilon(\sigma):=
d\bigl(g(\sigma),\,\mathcal I(\sigma)\bigr).
\]
G1 aggregates this discrepancy over the declared atomic population and holds when
\[
\varepsilon_{\text{pres}} \leq \bar\varepsilon_{\text{pres}}.
\]
A maximum requires every atom to meet the tolerance, while a quantile permits a declared fraction to exceed it. The discrepancy measure depends on the meaning type. Examples include cosine distance from a gold embedding, $1-\operatorname{IoU}$ for referential tasks, the discrete metric for logical meanings, and repair cost for social-normative content.

In the grid world, G1 takes the maximum over six atoms. The largest discrepancy is $0.098$, at \texttt{WEST}.

\paragraph{G2 (faithfulness)}
Faithfulness distinguishes agreement with the intended meanings (G2a) from evidence that the mechanisms producing it were retained for their contribution to success (G2b).

G2a aggregates the discrepancy $\varepsilon(\sigma)$ defined in G1 over the declared composite population and holds when
\[
\varepsilon_{\text{faith}} \leq \bar\varepsilon_{\text{faith}}.
\]
G1 and G2a thus apply the same discrepancy to atoms and composites, with separately declared aggregates and bounds. Systems without composites report G2a as not applicable. In the grid world, G2a takes the maximum over eight two-atom commands. The largest discrepancy is $0.580$, at the withheld \texttt{RED WEST}.

G2b concerns a designated agent-internal mechanism $M$ within $\Phi$ or $\mathcal A$. The audit fixes a success predicate $\text{succ}$ on task outcomes, such as reached positions, produced labels, or dialogue states, and imposes requirements:

\begin{enumerate}
    \item historical $M$ was retained under the declared process $\mathcal T$ because it improved $\text{succ}$. The process may be learning, evolution, or design, provided it differentially retains variants according to success.
    \item dispositional disabling $M$ in the evaluated  regime reduces success by at least a declared margin $\eta$.
\end{enumerate}
The dispositional requirement is measured by
\[
\text{ACE}(M):=
\mathbb E_{P^{\text{do}}}
\left[
\Pr\bigl(\text{succ}\mid\text{do}(M=\text{on})\bigr)
-
\Pr\bigl(\text{succ}\mid\text{do}(M=\text{off})\bigr)
\right],
\]
with averaging over $\mathcal T$ where relevant. G2b holds when $\text{ACE}(M)\geq\eta$ and the historical requirement is satisfied.

The intervention distribution $P^{\text{do}}$ is obtained from $P$ by resampling declared spurious covariates independently of success-relevant covariates. If none are declared, then $P^{\text{do}}=P$. This tests whether the mechanism's contribution depends on an incidental association between cue and success; G2a remains scored on $P$. The effect under decoupling may, however, match or exceed the effect under $P$, since the result also depends on the mechanism's other functions and the off-condition.

In the grid world, $M$ is the integration of the direction token, and $\text{succ}$ means reaching a position inside the success radius. With $P^{\text{do}}=P$ and $\eta=0.10$, the effect is $0.875$, so the dispositional requirement passes.

This intervention establishes present contribution. Historical certification additionally requires evidence about retention, assessed under an explicit model of selection and alternative retention routes. The historical status is:

\begin{itemize}
    \item certified when the evidence supports a contribution of the audited success property to retention;
    \item failed when it supports the absence of that contribution;
    \item uncertified when it is inconclusive between them.
\end{itemize}

Evidence from an intervention on the trained system leaves the history uncertified. A mechanism retained for another property can also have been retained for the audited one, since it may have several selected effects. Likewise, a mechanism selected for the audited property can persist after that selection pressure is removed: another pressure may suffice to retain it.

G0-strong and G2b, therefore, impose different conditions. A mechanism acquired through next-token training may well satisfy G0-strong and fail G2b for an extensional task. Conversely, a hand-written lexicon retained through design iterations because it improved task success may satisfy G2b while failing G0-strong.

G2b is also relative to the audited context, meaning type, and success predicate. Selection for linguistic prediction and selection for success in an embodied task are distinct historical claims. At a fixed audit specification, a shortcut may meet G2a while failing G2b; a mechanism retained for improving task success may meet G2b while its realized meanings remain outside the tolerance required by G2a.

\paragraph{G3 (robustness)}
Robustness measures how much realized meaning changes under perturbations. The threat specification identifies which variations should preserve the relevant content. Changes that alter the intended meaning, such as some shifts in context or register, require a corresponding target adjustment rather than an invariance test.

For representation-level threats, let $U$ contain perturbations $u:\mathcal R\to\mathcal R$ with declared scales $\varepsilon(u)$. Define the worst-case change at $r$ by
\[
D_{U,\varepsilon}(r):= \sup_{\substack{u\in U\\ \varepsilon(u)\leq\varepsilon}} d\bigl(\mathcal A(r),\,\mathcal A(u(r))\bigr).
\]
Assume that the identity belongs to $U$ at scale zero and that $D_{U,\varepsilon}$ is measurable. At coverage $1-\alpha$, with $0\leq\alpha<1$, define
\[
\omega_U(\varepsilon):= \inf\left\{ w\in[0,\infty): \Pr_{r\sim P}\bigl[D_{U,\varepsilon}(r)\leq w\bigr] \geq 1-\alpha \right\},
\]
taking $\inf\varnothing=\infty$. For surface threats, replace $r\sim P$ by $\sigma\sim P$, let $u$ act on $\Sigma$, and replace $\mathcal A$ by $g$.

The supremum ranges over admissible perturbations at each item; the quantile ranges over items. Thus $1-\alpha$ specifies population coverage. Sampled perturbations bound the per-item supremum from below. The covered items may also differ across scales.

The curve $\omega_U$ is nondecreasing and satisfies $\omega_U(0)=0$ when every zero-scale perturbation is the identity. G3 holds when
\[
\omega_U(\varepsilon)\leq\bar\omega(\varepsilon) \qquad \text{for every }\varepsilon\in[0,\varepsilon_{\max}],
\]
where the audit declares the range and bound in advance. A Lipschitz bound $\bar\omega(\varepsilon)=L\varepsilon$ is one option. A finite image diameter supplies a constant upper bound, which certifies G3 at scales where the declared bound is at least that diameter.

The audited curve differs from the global oscillation considered in Theorem~\ref{thm:modulus}; that global quantity vanishes as the scale tends to zero exactly when $\mathcal A$ is uniformly continuous. Corollary~\ref{cor:domination} bounds the audited curve by it when the perturbation scale controls representation distance. Small audited values remain compatible with a global oscillation bounded away from zero.

On a uniformly discrete representation space, every readout is uniformly continuous. If distinct representations are separated by at least $s>0$ and $d_{\mathcal R}(r,u(r))\leq\varepsilon(u)$, every admissible perturbation below scale $s$ fixes the representation. Above that scale, robustness depends on the readout and threats; a finite image diameter supplies only an upper bound. Surface edits must be assessed through the encoder, since a small edit may produce a large representational change.

In the grid world, the hidden-state threat contains all perturbations of norm at most $\varepsilon$. The linear readout gives the same worst-case displacement at every state, so the population quantile equals that supremum. At radius $0.5$, $\omega_U(0.5)=0.840$.

\paragraph{G4 (compositionality)}
G4 measures agreement with the declared composition rule and accuracy on novel combinations. These are distinct: a system may compose its own interpretations consistently, while missing the intended meanings.

Let the grammar provide typed constructors $f:\Sigma^{\vec\kappa}\to\Sigma^{\kappa'}$, with corresponding operations $f^{\mathcal M}$ in the semantic algebra. Define
\[
\delta_{\text{comp}}(f(\vec\sigma)):=d\Bigl(g(f(\vec\sigma)),\,f^{\mathcal M}(g(\sigma_1),\dots,g(\sigma_n))\Bigr).
\]
This compares the realized composite with the algebra applied to its realized parts, independently of the target interpretation $\mathcal I$. The compositionality requirement holds when the declared aggregate over constructor applications satisfies
\[
\delta_{\text{comp}}\leq\bar\delta_{\text{comp}}.
\]

Systematicity measures target accuracy on a nonempty population $\mathcal D_{\text{novel}}$ of composites excluded from training and calibration:
\[
\beta:=\frac{1}{|\mathcal D_{\text{novel}}|}\sum_{\sigma\in\mathcal D_{\text{novel}}}\mathbf 1\bigl[\varepsilon(\sigma)\leq\tau\bigr].
\]
Here, $\tau$ determines whether an item passes, and the acceptance rate $\bar\beta$ determines whether systematicity passes. G4 requires both $\delta_{\text{comp}}\leq\bar\delta_{\text{comp}}$ and $\beta\geq\bar\beta$.

Each measurement has its own population: atoms for G1, composites for G2a, constructor applications for $\delta_{\text{comp}}$, and withheld composites for $\beta$. These measurements can constrain one another. If G2a uses a maximum over a population containing $\mathcal D_{\text{novel}}$, with $\bar\varepsilon_{\text{faith}}\leq\tau$, passing G2a entails $\beta=1$. Reporting $\beta$ certainly still identifies performance on withheld combinations, but its acceptance condition adds no further restriction in that particular audit.

Pointwise zero compositional deviation throughout the syntactic closure makes $g$ a homomorphism (see Theorem~\ref{thm:hom}); together with exact atomic preservation, this yields agreement with the homomorphic extension of the atomic targets (see Corollary~\ref{cor:coincide}). When the composite targets follow that extension, composite accuracy follows as well. These conclusions require the coverage and exactness. The coordinates distinguish the comparisons being made; which conditions entail others depends on the audit specification and semantic algebra.

In the grid world, the semantic operation is vector addition. Across the eight two-atom commands, the maximum deviation is $0.062$ on the six trained pairs and reaches $0.670$ at \texttt{RED WEST}. Of the two withheld commands, one passes at $\tau=0.5$, giving $\beta=0.50$.

\paragraph{Profile}

For vectorial representations, the audit may additionally measure whether representational distances agree with distances between target meanings. On a declared distribution of pairs, define
\[
\varepsilon_{\text{geom}}(\sigma_i,\sigma_j):=\left|d_{\mathcal R}\bigl(\Phi(\sigma_i),\Phi(\sigma_j)\bigr)-d\bigl(\mathcal I(\sigma_i),\mathcal I(\sigma_j)\bigr)\right|.
\]
The audit specifies the scales of both metrics and compares a declared quantile with $\bar\varepsilon_{\text{geom}}$. This optional measurement supplements the grounding profile.

Table~\ref{tab:technical} summarizes the required declarations and measurements. Estimates carry uncertainty intervals where sampling or stochastic variation requires them. Exact finite evaluations are reported as such.

\begin{definition}
The \textbf{grounding profile} of $\mathfrak G$ at index $E$
under an audit specification consists of:
\begin{itemize}
    \item a membership record for each component relied upon and an acquisition record for each system component;
    \item the six measurements
    $\varepsilon_{\text{pres}}$,
    $\varepsilon_{\text{faith}}$,
    $\text{ACE}(M)$,
    $\omega_U(\cdot)$,
    $\delta_{\text{comp}}$, and $\beta$,
    reported against their declared bounds;
    \item the G2b historical status---certified, failed, or
    uncertified---with the evidence, controls, and assumptions
    supporting it.
\end{itemize}

\end{definition}

Measurements are reported with their estimation uncertainty where applicable. Inapplicable coordinates are marked as such. Geometric alignment may be appended when declared.

\begin{table}
    \centering
    \footnotesize
    \setlength{\tabcolsep}{4pt}
    \renewcommand{\arraystretch}{1.2}
    \caption{Audit measurements, declarations, and acceptance conditions.}
    \label{tab:technical}
    \begin{tabular}{
        >{\raggedright\arraybackslash}p{1.2cm}
        >{\raggedright\arraybackslash}p{4.5cm}
        >{\raggedright\arraybackslash}p{6.0cm}
        >{\raggedright\arraybackslash}p{2.5cm}}
        \toprule
        \textbf{Clause}
        & \textbf{Measurement}
        & \textbf{Declarations and assessment}
        & \textbf{Acceptance} \\
        \midrule

        \textbf{G0}
        & component membership and acquisition records
        & declare system boundary, history $\mathcal T$,
          and components relied upon; inspect their provenance.
          Record instrument components separately
        & weak: system membership;
          strong: own learning or evolution \\

        \textbf{G1}
        & atomic discrepancy $\varepsilon_{\text{pres}}$
        & declare target, metric, atomic population,
          and aggregate (maximum or quantile)
        & $\varepsilon_{\text{pres}}
          \leq\bar\varepsilon_{\text{pres}}$ \\

        \textbf{G2a}
        & composite discrepancy $\varepsilon_{\text{faith}}$
        & declare target, metric, composite population,
          and aggregate (maximum or quantile).
        & $\varepsilon_{\text{faith}}
          \leq\bar\varepsilon_{\text{faith}}$ \\

        \textbf{G2b}
        & causal effect $\text{ACE}(M)$ and historical
          certification status
        & declare $\text{succ}$, $M$, on/off intervention,
          and construction of $P^{\text{do}}$ from $P$;
          assess retention using stated historical controls
          and selection assumptions.
        & $\text{ACE}(M)\geq\eta$ and historical
          requirement satisfied \\

        \textbf{G3}
        & robustness curve $\omega_U(\varepsilon)$:
          $(1-\alpha)$-quantile over $P$ of worst-case
          semantic change at scale $\varepsilon$
        & declare metric, population $P$, perturbation location ($\mathcal R$ or $\Sigma$), threat family, scales, and coverage $\alpha$;
          report supremum or approximation used;
          sampled perturbation quantiles estimate
          the worst case from below
        & $\omega_U(\varepsilon)\leq\bar\omega(\varepsilon)$
          over declared range \\

        \textbf{G4}
        & compositional deviation $\delta_{\text{comp}}$ and withheld-composite accuracy $\beta$
        & declare metric, grammar, algebra, constructor applications, and deviation aggregate;
          for $\beta$, declare targets, withheld population, its overlap with G2a, and item tolerance $\tau$.
        & $\delta_{\text{comp}}\leq\bar\delta_{\text{comp}}$ and $\beta\geq\bar\beta$ \\

        \textbf{Optional}
        & geometric discrepancy $\varepsilon_{\text{geom}}$
        & declare targets, pair distribution, both metrics and their scales, and quantile to report.
        & $\varepsilon_{\text{geom}}
          \leq\bar\varepsilon_{\text{geom}}$ \\

        \bottomrule
    \end{tabular}
\end{table}

\subsection{Authenticity}\label{subsec:g0}

G0 records whether the audited mappings operate within the agent and how they were acquired. Membership distinguishes the agent's mechanisms from the auditor's instruments; acquisition distinguishes own learning or evolution from import and design. Both records underdetermine what the mappings mean, why they were retained, and whether the agent understands. These distinctions bear on original and derived intentionality and are neutral between them. Derived intentionality is aboutness attributed through an interpreter's practices or intentions; original intentionality is supposed to belong to the agent independently of that attribution. Note: the status of this distinction is very much contested; classic contestation follows from Searle, who treats it as a difference in kind \cite{searle1983intentionality}, and Dennett, who challenges the demand for an original intentionality underlying interpretive attribution at all \cite{dennett1987intentional}.

G0 addresses a narrower issue: which mechanisms can be credited to the system. An auditor-supplied readout remains part of the instrument. A hand-written lexicon or an encoder copied from another model may qualify as a system component, but its provenance supports at most G0-weak when its production lies outside the declared learning or evolutionary history. A learner that acquires its mappings through developmental feedback may satisfy G0-strong. So may a student trained on another model's outputs: optimization within its own history fixes its parameters, even though the training information comes from a teacher. The acquisition record preserves that source; G2b separately examines selection for the audited success property.

Own acquisition leaves the intentional status of these mappings unsettled (the Chinese Room challenges the sufficiency of formal symbol manipulation for understanding \cite{searle1980minds}). The challenge extends to a system that learned its rules, and internally implemented mechanisms can have interpreter-dependent content. A content theory may, nevertheless, give acquisition and selection a constitutive role, as teleosemantic accounts give explanatory weight to functional history \cite{Millikan1984MILLTA,dretske1988explaining}. Other accounts may treat such history as insufficient or irrelevant to understanding. The audit records membership and acquisition, and its verdict holds whichever way that disagreement is resolved.

G0, therefore, limits what a favorable audit can establish. When an analyst supplies an interpretation at audit time, the resulting accuracy is evidence about the instrument. Instrument-mediated results remain conditional on the stipulated adapter; claims about the agent's own mechanisms must meet the applicable provenance tier. This keeps attribution of the measured performance explicit while leaving its metaphysical significance open.

\subsection{Selection and etiology}\label{subsubsec:g2bselection}

G2b gives \emph{selection for} a historical meaning: a mechanism $M$ was retained under a process $\mathcal T$ because it contributed to success in the declared task family. This requirement concerns a particular mechanism and effect. G0-strong asks only whether the audited machinery was acquired through the system's own learning or evolution.

We adopt a generalized selected-effects account of function \cite{garson2019biological,garson2011selected}. On this account, a mechanism has the selected function of producing an effect when mechanisms of its type were differentially retained, because they produced that effect. The relevant history may involve organismal traits across generations, neural mechanisms during development, parameter configurations during training, or linguistic and artifact variants within a community. In each case, the audit must identify the variants, the retention process, and the effect that explains their retention.

This formulation permits G2b to hold where G0-strong fails. Designers may retain a hand-written mechanism because it improves task success, even though its acquisition status is designed. Conversely, a learned mechanism may contribute to present success while its retention has another explanation.

Present contribution and historical selection, therefore, require separate evidence. A by-product may pass the dispositional test while failing the historical requirement. Redundancy may produce the reverse pattern: a mechanism retained for an effect can fail the ablation test because another mechanism compensates for its removal. G2b requires both conditions. A large $\text{ACE}(M)$ establishes present contribution under the declared intervention.

Historical controls may include retraining with altered rewards or supervision, tracking ablations during training, and varying developmental or environmental conditions while controlling identified confounds. The audit states the historical hypothesis, the control, and the assumptions connecting the result to retention. These assumptions must address alternative routes by which the mechanism could arise or persist.

Whether optimization constitutes selection requires particular care. Gradient methods compute updates, so their relation to selection among variants requires argument. The distinction between selection and instruction motivates this concern \cite{godfreysmith1994modern}. Settling it requires more than a single trajectory: the audit must explain what counts as a variant, how variants are retained or replaced, and how the audited effect influences that process. Reducing a loss correlated with success may select for other properties the loss rewards.

The same requirement applies to design histories. A record of improving performance must be supplemented by evidence that the audited effect contributed to retaining the designated mechanism. Any certificate based on optimization or design states the model of retention it assumes. G2b remains uncertified until the available evidence establishes that history, even when accuracy and present causal contribution are high.

G2b evaluates selection relative to the declared success predicate. It takes the intended interpretation $\mathcal I$ as declared at the auditor level. This limits its ambition relative to teleosemantic accounts that use function to explain content \cite{Millikan1984MILLTA,dretske1988explaining}. The audit establishes, where the evidence permits, why a mechanism was retained for an effect. Whether that history constitutes meaning remains a question for the content theory.

\section{Worked audits}\label{sec:toy}

We first audit a grid-world navigator trained by REINFORCE. Meanings are coordinates in a continuous $10\times10$ plane, compared by Euclidean distance. Colors denote landmarks, directions denote unit vectors, and the declared composition rule is addition. The atomic population contains two colors and four directions; $P$ is uniform over their eight combinations. Two commands, \texttt{BLUE EAST} and \texttt{RED WEST}, were withheld from training. Success means reaching within $\tau=0.5$ of the target. See Appendix~\ref{sec:gridworld} for implementation, computations, and a supervised comparison.

Table~\ref{tab:g0g1g2g3g4} shows the audit. The causal-effect threshold is $\eta=0.10$. Acceptance bounds for accuracy, robustness, compositional deviation, and systematicity remain undeclared, so those measurements are descriptive. The finite populations are evaluated exhaustively; the hidden-state supremum is analytic, while the directional mean is sampled.

\begin{table}
    \centering
    \small
    \setlength{\tabcolsep}{4pt}
    \renewcommand{\arraystretch}{1.2}
    \caption{Grid-world profile. Measurements are descriptive
    unless an acceptance condition is stated.}
    \label{tab:g0g1g2g3g4}
    \begin{tabular}{
        >{\raggedright\arraybackslash}p{2.0cm}
        >{\raggedright\arraybackslash}p{7.0cm}
        >{\raggedright\arraybackslash}p{5.5cm}}
        \toprule
        \textbf{Clause}
        & \textbf{Result}
        & \textbf{Population or scope} \\
        \midrule
        \textbf{G0}
        & strong: system components acquired through training
        & embedding, GRU, and readout \\

        \textbf{G1}
        & maximum discrepancy $0.098$
        & six atoms \\

        \textbf{G2a}
        & maximum discrepancy $0.580$;
          success rate $0.875$
        & eight composites, including both withheld commands \\

        \textbf{G2b}
        & $\text{ACE}(M)=0.875\geq\eta=0.10$;
          historical status uncertified
        & directional integration ablated;
          eight composites under $P^{\text{do}}=P$ \\

        \textbf{G3}
        & hidden-state supremum $0.840$;
          sampled directional mean $0.095$
        & radius $0.5$;
          mean over $256$ sampled directions \\

        \textbf{G4: deviation}
        & maximum $0.062$ on trained pairs;
          $0.670$ on withheld pairs
        & six trained and two withheld
          constructor applications \\

        \textbf{G4: systematicity}
        & $\beta=0.50$ at $\tau=0.5$;
          $\bar\beta$ undeclared
        & two withheld composites \\
        \bottomrule
    \end{tabular}
\end{table}

The profile distinguishes the following aspects of performance. 

\begin{enumerate}
    \item A system can pair atomic and familiar-composite accuracy with a failed combination. As Figure~\ref{fig:redwest} shows, adding the realized meanings of \texttt{RED} and \texttt{WEST} gives a point inside the success region, while the realized composite lies outside it. The discrepancy identifies a departure from the declared algebra. Memorization, and repair by enforced addition, are hypotheses for further intervention.

    \item Disabling sequential integration beyond the first token eliminates all seven successful commands; the declared set of spurious covariates is empty, so $P^{\text{do}}=P$. The intervention establishes present contribution under that off-condition. Whether the audited success property contributed to retaining the mechanism remains historically uncertified.

    \item At hidden-state radius $0.5$, the linear readout permits a displacement of $0.840$, compared with a sampled directional mean of $0.095$. The mean describes the sampled directions, whereas G3 uses the per-item supremum before aggregating over items. Sampled perturbations can, then, understate the worst admissible change. 
\end{enumerate}

The word2vec pilot in Appendix~\ref{sec:pilot} adds a different contrast: its atomic encoder is learned, but its phrase pooling and evaluation readouts are supplied by the auditor; the resulting accuracy and perturbation measurements are conditional on those instruments. Permutation ablation supplies evidence of the learned assignment's present contribution to agreement with human ratings, while historical selection remains uncertified. Adjective split measures audit-held-out accuracy under stipulated pooling. Learned composition and exclusion from pretraining would require other evidence. Together, the audits show how a profile separates measured performance from the provenance and evidential limits of the mechanisms credited with it.

\section{Modes and cases}\label{sec:variground}

Grounding proposals differ in how they represent symbols and connect them to their targets. We distinguish representation format from world coupling: symbolic, vectorial, and relational formats describe how representations are organized; coupling describes their causal connection to the environment. A system may certainly combine formats, and any format may participate in sensorimotor interaction.

This distinction lets the framework assess different proposals and test whether their preferred mechanisms suffice for grounding. Connecting symbols to non-symbolic representations addresses the regress motivating Harnad's proposal \cite{HARNAD1990335}, but leaves their accuracy, composition, and acquisition open to evaluation. Likewise, correspondence between a text-induced geometry and worldly structure \cite{sogaard2023grounding} supplies a claim about alignment whose causal and historical basis requires separate evidence. The formats below provide audit specifications for these questions, drawing on the distinctions discussed by \cite{mollo2023vector}.

\subsection{Formats and coupling}\label{subsec:formats}

Symbolic representations support operations over discrete tokens; vectorial representations organize states in a vector space; relational representations encode typed objects and relations. These descriptions overlap. A knowledge graph may be implemented symbolically, and a learned system may combine vectors with explicit relational operations.

Here, \emph{referential coupling} denotes a sensorimotor connection between representations and the world, and is one possible basis for extensional grounding, whose adequacy still depends on the audit. A multimodal system may, however, be uncoupled in this sense: the relevant evidence concerns the provenance of the observations, the learning process, and the role of the resulting mechanisms in world-directed tasks.

Table~\ref{tab:modes} gives representative realizations and audit questions for investigation. Note that its contexts and threats are defaults, and each format admits others. Appendix~\ref{sec:modedefs} supplies the corresponding formal specifications.

\begin{table}
    \centering
    \footnotesize
    \setlength{\tabcolsep}{4pt}
    \renewcommand{\arraystretch}{1.2}
    \caption{Representation formats and referential coupling; realizations illustrate possible audit specifications, and grounding verdicts come from the audits themselves.}
    \label{tab:modes}
    \begin{tabular}{
        >{\raggedright\arraybackslash}p{2.0cm}
        >{\raggedright\arraybackslash}p{4.5cm}
        >{\raggedright\arraybackslash}p{3.4cm}
        >{\raggedright\arraybackslash}p{4.6cm}}
        \toprule
        \textbf{Format or coupling}
        & \textbf{Representative realization}
        & \textbf{Audit}
        & \textbf{Questions for measurement} \\
        \midrule

        \textbf{Symbolic}
        & discrete expressions with lookup or rule-based interpretation; simplest case, $\mathcal R=\Sigma$ and $\Phi=\operatorname{id}$
        & logical context, inferential meanings, surface-edit threats
        & atomic and composite accuracy; agreement with semantic algebra; sensitivity to edits; acquisition and retention of rules \\

        \textbf{Vectorial}
        & $\mathcal R=\mathbb R^n$ with declared metric, an encoder $\Phi$, readout $\mathcal A$
        & linguistic or visual contexts, with threats specified at input or representation level
        & target agreement, geometric alignment where declared, perturbation sensitivity, composition on novel combinations \\

        \textbf{Relational}
        & typed nodes, edges, or axioms, with graph metric and readout based on roles or rule closure
        & inferential meanings under edits to relations, types, or axioms
        & agreement with target relations; propagation of edits through inference; composition across typed structures \\

        \textbf{Referential coupling}
        & sensorimotor representations connected to world or task states through perception and action
        & extensional meanings in embodied context, under declared environmental perturbations
        & accuracy of world-directed interpretation, causal dependence on coupling, evidence of retention for task success \\
        \bottomrule
    \end{tabular}
\end{table}

Provenance is assessed component by component in every format. Designed (or imported) mappings support at most G0-weak; mappings acquired through the declared learning or evolutionary history may support G0-strong. An auditor-supplied readout remains part of the instrument. Note that these records are orthogonal to the choice of format and coupling condition, and the same applies to robustness and composition. A uniformly discrete representation space makes continuity vacuous below its minimum spacing. The first admissible edit may change meaning by any amount up to the image diameter. A vector space permits graded perturbations, whose semantic effects can be large. Rule-based systems compose exactly only when their realized operations respect the declared algebra. In every format, systematicity must be measured on valid withheld composites; expressions outside the grammar require a separate specification.

Which effects explain the retention of the audited mechanisms belongs to G2b. Stipulating a symbolic interpretation is inert with respect to selection history, although a documented process of design comparison may establish one. Learned rules or relations likewise receive G2b certification only when the evidence connects their retention to the declared success predicate.

Sensorimotor histories can certainly provide such evidence for world-directed tasks when successful interaction explains retention; mere exposure to sensory data, however, falls short of establishing that history. For abstractions, theoretical entities, or fictional kinds, the auditor must also justify the chosen meaning space. Depending on the claim, inferential or social-normative standards may be more appropriate than extensional ones.

For vectorial systems, geometric regularities provide evidence about representational organization. An analogy pattern such as the classic
\[
\Phi(\text{king})-\Phi(\text{man})+\Phi(\text{woman}) \approx\Phi(\text{queen})
\]
describes that organization. Why the organization arose, and what success property retained it, are historical questions. Linguistic training may support an etiological claim about linguistic performance; extending that claim to world-directed success requires evidence connecting the training history to that success.

Image-text contrastive training, as in CLIP \cite{Radford2021LearningTV}, combines vectorial representations with cross-modal alignment. An audit must distinguish that alignment from sensorimotor coupling and selection for action. Systems combining learned and relational components \cite{shengyuan2023differentiable,bosselut2021dynamic} likewise require separate provenance and causal assessments for the components credited with performance.

The following cases apply these distinctions: model-theoretic semantics illustrates stipulated symbolic interpretation; natural language is assessed across extensional, inferential, and social-normative uses; language models are examined through their learned representations, training histories, and available readouts.

\subsection{Cases}\label{sec:scrutinizing}

Model-theoretic semantics and ordinary human language provide contrasting reference cases, insofar as the first tests whether the framework distinguishes exact stipulated interpretation from evidence of acquisition and selection, and the second tests its fit with the communicative practices that motivate grounding claims to begin with. These comparisons serve a calibration role \cite{chang2004inventing}, but their status differs: properties of the stipulated formal system follow by construction, whereas the human profile requires empirical assessment.

Table~\ref{tab:calibration} separates those formal consequences from expectations about human language. Agreement with an expected profile supports the framework's conceptual fit; disagreement requires examining the framework, the audit specification, and the expectation itself. Language models provide a further case in which the grounding verdict is contested.

\begin{table}
    \centering
    \footnotesize
    \setlength{\tabcolsep}{4pt}
    \renewcommand{\arraystretch}{1.2}
    \caption{Reference profiles for a stipulated formal interpreter
    and ordinary human language use. The human column states
    expectations and evidential requirements; its entries are unmeasured.}
    \label{tab:calibration}
    \begin{tabular}{
        >{\raggedright\arraybackslash}p{1.5cm}
        >{\raggedright\arraybackslash}p{6.2cm}
        >{\raggedright\arraybackslash}p{6.2cm}}
        \toprule
        \textbf{Clause}
        & \textbf{Stipulated formal interpreter}
        & \textbf{Human language use} \\
        \midrule

        \textbf{G0}
        & weak when designed interpretation is implemented within declared system
        & strong where audited mechanisms were acquired within a history including evolution and development \\

        \textbf{G1}
        & zero discrepancy on atoms in interpretation's domain
        & agreement depends on context, sense resolution, and declared standard \\

        \textbf{G2a}
        & zero discrepancy on well-typed composites interpreted by same semantic algebra
        & agreement must be assessed across declared constructions and communicative situations \\

        \textbf{G2b}
        & historical requirement fails in stipulated case, whose construction excludes selection for task success
        & acquisition and communicative feedback motivate selection hypotheses; certification requires mechanism- and task-specific evidence \\

        \textbf{G3}
        & every edit below minimum spacing is trivial; behavior at larger scales depends on interpretation and treatment of malformed inputs
        & tolerance of benign variation is expected locally; its extent depends on threat and population \\

        \textbf{G4}
        & zero compositional deviation; any nonempty valid withheld set has $\beta=1$ where $g=\mathcal I$
        & productive interpretation motivates systematicity tests; idioms and coercion require an appropriate grammar and algebra \\
        \bottomrule
    \end{tabular}
\end{table}

\subsubsection{Model-theoretic semantics}\label{sec:modground}

We evaluate a stipulated implementation of model-theoretic semantics at inferential coordinates in a logical context \cite{zimmermann2019model,garson2013logics}. System boundary and acquisition history belong to implementations of an abstract semantics. Here those coordinates apply to an interpreter implementing it.

Let $\mathcal R=\Sigma$ with an edit metric whose minimum nonzero distance is $1$, and let $\Phi=\operatorname{id}_\Sigma$. An analyst supplies the interpretation
\[
\mathcal I^{\text{inf}}: \Sigma\rightharpoonup\mathcal M^{\text{inf}},
\]
which assigns denotations in a mathematical structure. The interpreter implements this same map as its readout
$\mathcal A^{\text{inf}}$. On its domain,
\[
d\bigl( g^{\text{inf}}(\sigma), \mathcal I^{\text{inf}}(\sigma) \bigr)=0.
\]

Atomic accuracy, therefore, holds by construction. When both maps extend through the same typed semantic algebra, composite accuracy and compositional deviation are also zero throughout the interpreted closure. Any valid withheld composites in that domain consequently pass the systematicity test.

In the case constructed here, the interpreter is designed and installed within the system, so it satisfies G0-weak. We additionally stipulate that the construction of its mechanisms excludes selection for the audited task success; the historical requirement of G2b, therefore, fails. This failure belongs to the constructed history. A designed interpreter retained through task-based comparisons would require a different etiological assessment.

Discreteness also has a limited implication: below radius $1$, the edit ball contains only the original expression, so semantic change is zero; at the first admissible edit, the change depends on the denotations assigned to neighboring expressions, which can be large and may stay below the image diameter. If an edit produces an expression outside the interpretation's domain, then the audit must specify how that outcome is evaluated. Uniform continuity on the discrete domain concerns edits that stay within it.

The case illustrates the distinction between formal interpretation and world-directed content raised by \cite{peregrin1997language}. Assigning \textit{cat} a denotation in a mathematical structure underdetermines how that denotation relates to actual cats. The further grounding question remains even when every formal comparison is exact \cite{gubelmann2024pragmatic}.

Classical compositional semantics, thus, supplies a precise structural reference point: zero discrepancy at every atom and zero deviation at every constructor application recover the intended homomorphic interpretation (see Theorem~\ref{thm:hom}; Corollary~\ref{cor:coincide}). Acquisition, selection, and robustness require their own specifications and evidence.

\subsubsection{Natural language}\label{sec:nlground}

Ordinary human language provides a benchmark for grounding, although accounts disagree about what sustains its meaning. Relations to extra-linguistic reality are central to the account of \cite{benderkoller2020climbing}; inferential commitments and social practice organize the account of \cite{brandom1994making,brandom2010between}. The framework here represents these different concerns through extensional, inferential, and social-normative coordinates, with each assessment indexed to a context.

Persistent failures under well-supported standards and ordinary communicative conditions would challenge the framework's conceptual fit. An uncertified coordinate, however, may reflect missing evidence rather than a failure of human grounding. Table~\ref{tab:calibration}, therefore, states expectations to be tested.

The audit population consists of declared communicative situations, in which threats may include accent variation, channel noise, and typographical errors, with lighting, occlusion, and viewpoint changes for perceptual tasks in particular. Pragmatic changes require particular care: irony cues or register shifts may change the intended meaning. They test robustness only insofar as the declared standard treats the relevant content as preserved; otherwise they require evaluation at the changed context.

For G0, the relevant mechanisms connect speech and text with perception, action, memory, and social practice. A history including evolution and development provides a basis for G0-strong, subject to identifying the components relied upon. For G1 and G2a, ordinary communicative success motivates an expectation of substantial agreement with contextual targets. Ambiguity and polysemy complicate that comparison. Discourse and repair can help resolve them, but the audit must specify whether it measures an initial interpretation or the outcome of an interaction.

Developmental learning and social feedback also certainly motivate etiological hypotheses. Children encounter correction and environmental counter-evidence when their category use diverges from the surrounding practice \cite{t2010references,barbu2013language}. Such histories can be examined for retention of mechanisms supporting accurate reference, inferentially appropriate responses, or norm-appropriate contributions. Certification of G2b, nevertheless, requires evidence connecting a designated mechanism's retention to the particular success predicate, together with the dispositional test.

Multimodal redundancy and pragmatic repair support robustness under ordinary communicative variation \cite{winter2014spoken,winter2012robustness}, in which their effectiveness is local to the task and conditions: a small alteration can still change an utterance's interpretation substantially. G3, therefore, asks how much semantic change occurs under the declared threats and with what population coverage.

Productive interpretation motivates G4: speakers can interpret novel compounds, nested clauses, and morphological combinations. Their interpretations of novel items may depart from a particular semantic algebra. Expressions such as ``quantum fridge'' and ``purple Tuesday'' may admit several contextual interpretations, surely, so the audit must specify their targets and composition rules. Idioms and coercion may be represented as specialized constructions, whose adequacy remains subject to testing.

Further, the meaning type depends on the claim being assessed. Consider: ``justice'' may be examined through inferential commitments or social norms; ``quark'' can support extensional assessment despite indirect observational access; ``Sherlock Holmes'' may be evaluated relative to a fictional discourse. Direct perception is, therefore, one possible basis for a target, and one input to the choice of meaning type.

Cultural and lexical variation further affects the standard and its metric, which we may see in the object-derived color descriptions of Y\'el\^i Dnye \cite{levinson2000yeli} and the river vocabulary of Hobongan \cite{perkins2024colors}, which motivate standards attentive to local distinctions. Such differences may alter what the audit measures while the degree of grounding stays the same.

The expectation is that realized interpretations often agree with well-supported contextual targets:
\[
g(\sigma)\approx\mathcal I(\sigma).
\]
The size and distribution of that discrepancy, its stability, and the causal history supporting it, remain separate empirical questions.

\subsubsection{Large language models}\label{sec:llmground}

The grounding dispute concerns what learned linguistic behavior establishes about meaning. Bender and Koller distinguish learning linguistic form from learning its relation to communicative intent \cite{benderkoller2020climbing}; Harnad retains the requirement that symbols connect to non-symbolic resources \cite{harnad2024language}. Other accounts give distributional structure, conceptual roles, or communicative practice a larger role in grounding \cite{gubelmann2024pragmatic,piantadosi2022meaning}. Here, we separate the evidence these positions invoke from their claims about what constitutes meaning.

We use two illustrative indices. At $E_{\text{ling}}$, the context is linguistic, the meaning type inferential, and the population a declared set of linguistic items or tasks. Threats include specified text and prompt perturbations. At $E_{\text{vl}}$, the context is vision-language, the meaning type extensional, and the population consists of declared image-text instances. Perturbations must preserve the relevant target, or the audit must specify how that target changes.

\paragraph{Provenance and causal history}
A learned encoder $\Phi$ may satisfy G0-strong when the declared system history includes its training. An evaluation readout (supplied by the auditor) remains part of the instrument. The distinction is component-specific: in the distributional pilot, the atomic encoder is learned, while the composite extension and readout are stipulated (see Appendix~\ref{sec:pilot}).

A component's provenance is invariant under a change of meaning type. The change may affect which components the audit requires. An extensional readout implemented within the model must therefore be distinguished from one fitted solely for evaluation. The label ``text-only'' is compatible with its existence and with its absence. Training objectives also require such careful interpretation: a linguistic objective identifies what optimization directly rewards. Many capacities may contribute to that reward. G2b asks whether the particular audited effect helped explain retention of a designated mechanism. A mechanism supporting reference or inference might contribute to prediction; establishing that contribution historically requires more than identifying the objective.

Consider, for example, locating a coreference mechanism, observing its emergence during training, and showing that ablation lowers coreference accuracy would establish acquisition and present contribution. Certification of G2b would additionally require evidence that its contribution to the audited success property explains its retention. Preference and instruction tuning raise the same question for the outcomes rewarded by their feedback.

Externalist accounts introduce a further distinction, in that the argument that training strings participate in histories connecting words to the world \cite{mandelkernlinzen2024refer} concerns the model's possible reference, over and above an auditor's choice of standard. Certifying G2b for a designated mechanism and task needs mechanism-level evidence in addition to that claim; likewise, distinguishing linguistic tasks from tasks concerning the non-linguistic world \cite{butlin2023sharing} helps specify the success predicate, but leaves its causal and historical assessment to the audit.

\paragraph{Accuracy, robustness, and composition}
G1 evaluates atoms against the declared target interpretation. A quantile-based audit
\[
\Pr_{\sigma\sim P_{\operatorname{atom}}} \left[ d\bigl(g(\sigma),\mathcal I(\sigma)\bigr) \leq\bar\varepsilon_{\text{pres}} \right] \geq 1-\alpha,
\]
where $P_{\operatorname{atom}}$ and $\alpha$ are declared. G2a applies the discrepancy to a separately specified composite population. Choosing a quantile permits an explicit fraction of larger errors. The audit must judge whether those errors matter.

Embedding similarity protocols examine another relation: whether distances between representations agree with semantic judgments \cite{almeida2019word,selva2021review}. This bears on the optional geometric-alignment measurement: it establishes G1 or G2a only when a declared readout and target protocol connect that geometry to item-level discrepancies. A geometrically organized representation can still yield inaccurate realized meanings under a particular readout.

Multimodal contrastive training supplies additional evidence through image-text alignment. Such alignment can support extensional accuracy under an appropriate evaluation protocol; its role in G2b depends on the predicate: retention for matching images and captions is a different historical claim from retention for successful manipulation or navigation. The former may support mechanisms useful for the latter. Their selected function depends on which of these successes explains their retention.

Studies of token noise and orthographic variation examine particular forms of resilience \cite{wang2024resilience,moradi2021evaluating}; adversarial prompts pose different tests. Over a finite vocabulary, the space of arbitrarily long sequences is infinite, and the semantic image may be unbounded. Where the representation space is uniformly discrete, continuity below its minimum spacing is vacuous. Useful G3 claims therefore require explicit perturbations, scales, populations, and coverage. In this way, robustness is similarly specific to the threat. 

Compositional evaluation must distinguish accuracy on novel expressions from agreement with a semantic algebra. Evidence of targeted productivity \cite{xu2024largelanguagemodelscompositional} can coexist with failures on structured generalization tasks \cite{russin2024fregechatgptcompositionalitylanguage}. SCAN \cite{lake2018generalization}, COGS \cite{kim2020cogs}, and the tests of \cite{hupkes2020compositionality} illustrate ways to withhold combinations systematically. Their results depend on the model, training exposure, and split. The pilot's $\hat\beta=0.41$ concerns only a distributional baseline under its tier-ordering protocol.

\paragraph{Interpreting the profile}

These distinctions support profiles for specified systems and tasks, rather than a verdict for language models as a class. Indeed, a system may have internally learned representations, accurate readouts on one population, limited robustness under a particular threat, and uncertified selection history. Moving from inferential to extensional coordinates changes the targets and success predicates, so evidence must be reassessed.

The relation between linguistic form and communicative intent emphasized by \cite{benderkoller2020climbing} differs from the audit's selected-effects condition. A measurement of G2b, therefore, addresses part of their objection. Conversely, evidence of distributional or conceptual-role structure supports a claim about meaning only together with a content theory on which such structure suffices. A pragmatic account may regard that evidence as sufficient for some linguistic content while allowing further requirements for other forms of grounding.

The proposal that preference-based fine-tuning can confer world-involving functions \cite{mollo2023vector} yields a more specific empirical question: did feedback about the relevant world-directed success contribute to retaining the audited mechanism? Improved behavior or accuracy supplies part of the evidence, but historical certification requires the further retention argument. Once these empirical claims are separated, the remaining disagreement concerns their significance for meaning.

\paragraph{A discriminating hypothesis}
The framework suggests a comparison between a contrastively trained system $C$ and an interactively trained system $I$. Suppose they are matched within a declared tolerance on extensional G2a at $E_{\text{vl}}$; both are then evaluated on the same world-directed task, with a common success predicate, comparable output interfaces, and specified mechanisms and off-conditions. Here, comparability on this task rests on these shared specifications as well as on matched image-text accuracy.

Let $P$ be its reference distribution and $P^{\text{do}}$ the distribution obtained by decoupling declared caption-correlated cues from success-relevant variables. Write
\[
A_j(Q):=\text{ACE}_{Q}(M_j), \qquad D_j:=A_j(P)-A_j(P^{\text{do}}), \qquad j\in\{I,C\}.
\]
Here, $A_j(P^{\text{do}})$ measures present contribution after decoupling, and $D_j$ measures its reduction from the reference distribution. Either reduction may be negative.

The hypothesis is that interaction produces greater causal contribution after decoupling and less dependence on the caption-correlated cues:
\[
A_I(P^{\text{do}})-A_C(P^{\text{do}}) \geq\Delta_{\text{effect}}, \qquad D_C-D_I\geq\Delta_{\text{drop}},
\]
for positive margins declared before evaluation. These are separate predictions! Each goes beyond what the training objectives entail. Their assessment requires comparable ablations, controls for the mechanisms' other functions, and uncertainty estimates sufficient to distinguish a shortfall from an inconclusive result.

Retraining while altering the dependence of reward or supervision on world-task success may reveal whether that success contributed to a mechanism's emergence or retention. The inference requires controls for other effects of the intervention and alternative retention routes. Persistence after removing one pressure is inconclusive when another pressure can retain the mechanism. The comparison, therefore, tests both present dependence and
selection history. It may reveal similar extensional accuracy supported by different mechanisms and histories. Any resulting repair proposal (changed supervision; interaction; robustness
training; compositional constraints) must then be evaluated against the coordinate it is intended to improve.

\section{Discussion and limitations}\label{sec:discussion}

The framework's value, as with any, depends on whether its distinctions support informative measurements and useful diagnoses. The worked audits illustrate that use while discussing limits in certification, comparison, and repair.

\subsection{Profiles as diagnostics}\label{subsec:typology}

A profile describes an outcome under a declared index and audit specification; if we change the target, population, threat, or tolerance, we may well change the diagnosis. The coordinates distinguish diagnostic questions, although their answers may be linked by the audit specification and semantic algebra; a profile records both the measured differences and any dependencies warranted by those assumptions.

Correlational and etiological faithfulness distinguish accuracy from evidence about its mechanisms; Table~\ref{tab:g2ag2b} separates their possible verdicts. Its rows require a declared G2a acceptance bound; G2b is certified when both its dispositional and historical requirements are supported, failed when at least one is established to fail, and otherwise uncertified.

\begin{table}
    \centering
    \small
    \setlength{\tabcolsep}{4pt}
    \renewcommand{\arraystretch}{1.2}
    \caption{Correlational and etiological outcomes under
    a fixed audit specification. Uncertified G2b records
    insufficient evidence to establish satisfaction or failure.}
    \label{tab:g2ag2b}
    \begin{tabular}{
        >{\raggedright\arraybackslash}p{2.0cm}
        >{\raggedright\arraybackslash}p{4.1cm}
        >{\raggedright\arraybackslash}p{4.1cm}
        >{\raggedright\arraybackslash}p{4.1cm}}
        \toprule
        & \textbf{G2b certified}
        & \textbf{G2b uncertified}
        & \textbf{G2b failed} \\
        \midrule
        \textbf{G2a passes}
        & target agreement with etiological support
        & target agreement; etiology unresolved
        & target agreement despite established etiological deficit \\
        \textbf{G2a fails}
        & target error despite etiological support
        & target error; etiology unresolved
        & target error and established etiological deficit \\
        \bottomrule
    \end{tabular}
\end{table}

Passing G2a while failing G2b can indicate reliance on a shortcut. The combination has other possible causes. Historical selection may, instead, fail, despite present contribution, or redundancy may prevent the designated mechanism from meeting the ablation threshold. Conversely, a mechanism may satisfy G2b while realized meanings miss the G2a tolerance: contribution to task success and the required semantic accuracy can come apart. Either pattern admits several repairs.

The uncertified column records an evidential gap. Both executed audits provide evidence of present contribution while leaving historical selection unresolved. The pilot requires evidence connecting its training history to the audited rating-concordance predicate; the grid world requires evidence that success within the radius contributed to retention, beyond selection for proximity. The proposed historical controls remain unexecuted. Moreover, an undeclared G2a bound prevents assignment to either row: the grid-world results remain descriptive on that coordinate.

The remaining coordinates distinguish further diagnostic patterns in Table~\ref{tab:arch}. These patterns can overlap and form a partial list. They identify what to investigate next. Attributing a fixed character or a mechanism to an architecture requires evidence beyond its scores.

\begin{table}
    \centering
    \small
    \setlength{\tabcolsep}{4pt}
    \renewcommand{\arraystretch}{1.2}
    \caption{Diagnostic patterns under declared audit
    conditions. Each row specifies only the coordinates
    relevant to the contrast; the others remain open.
    Suggested tests are diagnostic; whether a repair works is settled by evaluation.}
    \label{tab:arch}
    \begin{tabular}{
        >{\raggedright\arraybackslash}p{3.0cm}
        >{\raggedright\arraybackslash}p{5.4cm}
        >{\raggedright\arraybackslash}p{5.8cm}}
        \toprule
        \textbf{Pattern}
        & \textbf{Profile}
        & \textbf{Diagnostic question} \\
        \midrule

        \textbf{Exact composition, inaccurate atoms}
        & G1 fails while $\delta_{\text{comp}}=0$. Composite accuracy and $\beta$ require separate assessment
        & Does changing the atomic interpretation improve target agreement while preserving composition? Atomic errors may cancel or accumulate. \\

        \textbf{Accurate composites, divergent composition}
        & G2a passes while $\delta_{\text{comp}}>\bar\delta_{\text{comp}}$ on evaluated applications.
        & Is the declared algebra appropriate? If so, what mechanism supports accuracy despite departure from it? \\

        \textbf{Familiar accuracy, withheld failure}
        & G2a passes on familiar composites, but
          $\beta<\bar\beta$ on strict withheld population.
        & Does failure involve inaccurate constituents, departure from the algebra, or both? Memorization is one hypothesis among several. \\

        \textbf{Accuracy with perturbation sensitivity}
        & G2a passes on clean composites while G3 fails
          under declared threat and bound.
        & Where does the perturbation enter, and which mechanism amplifies its effect? Front-end correction and readout control address different sources. \\

        \textbf{Instrument-dependent agreement}
        & Accuracy is measured through an auditor-supplied readout or composer, recorded as instrument under G0.
        & Does the system implement a corresponding interpretation or operation?
          Accuracy measured through an adapter is credited to the instrument. \\
        \bottomrule
    \end{tabular}
\end{table}

The grid-world audit illustrates the familiar--withheld contrast at the descriptive level, since its acceptance bounds are undeclared; its failed command also departs from the additive rule (although composing the realized atoms would fall inside the success region). This motivates testing compositional training or an architectural constraint on a fresh split. The pilot illustrates instrument-dependent agreement: its composite results concern learned atoms evaluated through stipulated pooling and readouts.

A pattern may suggest an intervention, but improvement must be measured under a comparable specification. Changing the success predicate or relaxing a tolerance changes the claim being assessed. Moving from binary success to a distance-based criterion, for example, requires recalculating the causal effect and reconsidering the historical evidence.

The strong grounding claim requires G0-strong and satisfaction of the remaining clauses at their declared tolerances (by Definition~\ref{def:grounding}). This is an acceptance condition within the framework, more modest than an unqualified attribution of original intentionality. A descriptive or partly uncertified profile can still guide investigation. How interventions change these profiles remains an empirical question.

\subsection{Standards, metrics, and thresholds}\label{subsec:metricization}

An audit must justify: the target interpretation; the measure of discrepancy; the tolerance used for acceptance.

The metric assumption requires care, even for familiar measures. For nonzero vectors, cosine dissimilarity $1-\cos$ generally fails the triangle inequality; angular distance or Euclidean distance between normalized vectors provides a metric alternative. Likewise, $1-\operatorname{IoU}$ has a metric interpretation for suitable set representations, but additional penalties require their own justification. A discrepancy measure that is not a (pseudo)metric may still support descriptive comparisons. Results relying on metric properties apply to it once the relevant properties are shown.

Social-normative content presents a less settled measurement problem. Candidate spaces include dialogue commitments and normative statuses; candidate discrepancies include repair cost, rater judgments, and distances in graphs of obligations. These choices encode different views of semantic proximity. Repair cost, for example, may be asymmetric, while rater scores may violate the triangle inequality. The audit must state which properties its measure supports and what its proxy task leaves unmeasured.

On Brandom's account, commitments and entitlements are tracked through social scorekeeping \cite{brandom1994making}. An audit, then, needs must operationalize part of that account by fixing a community and context, taking $\mathcal I^{\text{soc}}(\sigma)$ to specify the normative statuses associated with a competent use of $\sigma$; the realized interpretation assessed against that profile under a separately justified discrepancy measure. Agreement would support conformity to the operationalized practice, with the remaining grounding clauses assessed independently.

Where scorekeepers disagree, a single target may conceal relevant variation. The audit can, instead, report results against a family of standards indexed by perspective. Extensional standards likewise require choices about which worldly conditions are relevant and how they are measured. Disagreement can, therefore, concern the meaning type, the target at a fixed index, or the metric. Declaring these choices makes the disagreement inspectable. Resolving it is a matter of argument about the standards.

Acceptance thresholds pose a further problem. Estimates describe performance; bounds such as $\bar\varepsilon_{\text{pres}}$, $\bar\varepsilon_{\text{faith}}$, and $\bar\delta_{\text{comp}}$ specify what performance is acceptable. Bounds require a justification tied to the consequences of error and should be fixed before inspecting the results used for the verdict. We supply the machinery, but what that machinery is set to requires further justification. Indeed, thresholds are comparable only where the standards, metrics, scales, and populations support comparison in the first place. Across applications, a numerically smaller bound may impose the weaker requirement. Conventional idioms also call first for an appropriate grammar and semantic algebra; relaxing the compositionality bound may merely conceal a poor specification.

The framework reports measurements, uncertainty, and acceptance conditions separately; an application must still justify which deficits it can tolerate and which preclude use. Establishing defensible standards and acceptable regions of profile space, especially for social-normative content, remains a substantive part of applying the audit.

\subsection{Context and composition}\label{subsec:indexing}

Indexing by $(k,t)$ makes an audit precise only insofar as those indices are specified, so one word may fall under several meaning types: ``marriage'' identifies a legal and social status, supports inferences about rights and obligations, and carries normative consequences. An audit must state which of these aspects its target represents.

Similarly, besides the task domain, contexts $k$ may include interlocutor identity, discourse history, cultural background, comparison class, location, and salience. ``Tall'' depends on a comparison class, ``every student'' on a restricted domain, and ``here'' on the context of utterance. These dependencies may be represented by variation in $\mathcal I_k^t$ or incorporated into the meaning space itself; either choice requires specifying what remains fixed and what varies across the evaluated items.

At the same time, vagueness may persist within a fixed context. A point-valued interpretation may also assign, in place of a fully determinate extension, a set of admissible precisifications, a graded membership function, or another structured object. These alternatives express different theories of indeterminacy and require different comparison measures. Probability distributions represent uncertainty, and identifying them with vagueness requires further argument. The present audits set these choices aside and apply most directly where the target interpretation is sufficiently settled for the intended comparison.

G4 assumes typed constructors and a semantic algebra $\langle\mathcal M,\mathcal F\rangle$. An auditor may supply these to test a system whose grammatical organization is implicit. The resulting deviation measures agreement with that declared analysis. The system may achieve that agreement through internal constructors of its own. Any fitted composition mechanism used by the audit must also be recorded as part of the instrument.

Conventional expressions make the choice of algebra consequential; consider ``kick the bucket'' and ``red herring'', whose ordinary idiomatic meanings may come from sources other than an algebra designed for their literal uses. A construction-sensitive grammar can represent those meanings through specialized constructions, as motivated by constructional approaches \cite{BOAS2007,goldberg2006constructions}. Simply increasing the tolerance would obscure whether a large deviation reflects system failure or an unsuitable analysis.

Specialized constructions must, nevertheless, be justified independently of the outputs being scored. An algebra fitted to reproduce each observed expression could make compositional deviation small by stipulation. G4, therefore, retains separate measurements of agreement with the algebra and accuracy on withheld combinations. Reporting these by construction family can further distinguish productive behavior from success on conventional expressions.

Finally, discreteness limits what continuity can establish. Below the minimum spacing, a uniformly discrete representation space admits only trivial perturbations, while at larger scales, semantic change depends on the readout and threat; a finite image diameter supplies an upper bound on the first edit's effect. An audit of surface noise must, therefore, include the front end that carries those edits into representations. The resulting modulus describes sensitivity, while identifying an effective repair requires further intervention.

\subsection{Etiological access}\label{subsec:etiological}

G2b requires evidence of present causal contribution and historical selection. Both are difficult to obtain, even when a system's weights are available.

The dispositional test first requires identifying a mechanism $M$, such as a head, subspace, circuit, or distributed parameter set; its ablation must, then, distinguish the relevant contribution from disruption of other functions. Redundancy may conceal that contribution, while a broad intervention may exaggerate it. The reported causal effect is conditional on the mechanism and off-condition specified by the audit. Historical certification requires evidence that the audited effect contributed to retaining $M$. Training records, checkpoints, and controlled retraining can support that inference, but their interpretation must address alternative retention routes. Such evidence can be costly to obtain, and may be unavailable for proprietary systems.

When only present contribution can be estimated, the audit reports the effect and its uncertainty, while leaving the historical requirement uncertified. G2a is assessed separately against its declared bound. In this evidential state, historical selection and its absence are both possible. An uncertified G2b prevents certification of the strong grounding claim. The system may yet meet that claim, and its accuracy may have a systematic basis. Whether the available evidence supports a weaker attribution of grounding depends on the account being applied; the profile states precisely what remains unresolved.

\subsection{Temporal scope and symbol emergence}\label{subsec:temporality}

We evaluate a system at a specified time, while drawing on its history for provenance and etiology. Describing how profiles change through development, continued training, or shifts in linguistic practice is a task for later work.

A natural direction to take is repeated audits, provided they distinguish changes in the system from changes in the standard and evaluation population. A decrease in $\varepsilon_{\text{faith}}$ after domain adaptation indicates improved agreement only relative to the targets and items being compared. Likewise, maintaining agreement with a changing practice may require the system's interpretations to change. Temporal stability must accordingly allow some mappings to change.

Translation between representations raises a related issue: a map between two representation spaces may preserve realized meanings when it approximately commutes with the encoders and readouts. Under a shared target and metric, the resulting semantic error bounds the change in item discrepancy. Transporting compositional bounds additionally requires control of the semantic operations; transporting robustness requires control of metric distortion and correspondence between threats. Appendix~\ref{sec:translation} examines these conditions and a possible experimental comparison.

A translation learned within the declared history may satisfy the acquisition requirement, whereas copied or externally supplied components receive different provenance records. In either case, whether the translated mechanism was retained for the audited effect is a separate question. That historical claim, together with its present causal contribution, must be established for the resulting system.

And now for an especially intriguing puzzle: we begin with symbols already individuated. Embodied and enactive approaches examine how meaningful distinctions arise through organism-environment interaction \cite{vare1991,dipaolo2018}. The present audit presupposes a domain $\Sigma$ on which encoding, interpretation, and composition can be assessed. It takes as given that continuous interaction has become organized into stable symbolic units. This restricts application to systems with identifiable tokens or other manipulable forms, and is neutral on their ontological status. G0 records the provenance of the relevant machinery, and G2b examines its selection history. Either clause can hold for a system whose connections are purely symbolic. A sensorimotor or other non-symbolic connection requires its own specification and evidence.

An extension to symbol emergence would need to identify how candidate units are formed, stabilized, and reused in interaction. Continuous dynamics could provide a representation space, and emergent distinctions could supply a target for comparison. But relabeling those spaces would leave the central problem untouched unless the formation of $\Sigma$ itself were modeled. The current framework evaluates symbols once identified; their emergence remains for future work.

\section{Conclusion}\label{sec:concl}

The framework evaluates grounding claims through profiles of accuracy, provenance, causal and historical evidence, robustness, and composition. Each profile is relative to a declared semantic standard and audit specification. These distinctions make clear what a successful test supports and what remains unresolved.

The worked audits demonstrate the diagnostic value of such a separation. In the grid world, accurate atomic interpretations coexist with a failed combination that departs from the declared composition rule. Both that audit and the distributional pilot establish a designated mechanism's present contribution to success while leaving its historical selection uncertified. The missing evidence would come from historical controls.

Stipulated compositional semantics supplies an exact structural limit; human language motivates expectations that require measurement. For language models, changing the meaning type or task changes the evidence required. The proposed comparison of contrastive and interactive training tests whether similar semantic accuracy can accompany different causal dependencies and selection histories.

We set aside how meaning arises a priori; provenance requirements constrain which mechanisms receive credit. Original intentionality, and satisfaction of the zero semantical commitment condition, fall to the accounts under evaluation. Further limits concern the justification of semantic metrics and thresholds, access to selection histories, and the treatment of change and symbol emergence. The next step is comparative evaluation under shared standards, populations, and interventions, together with historical controls where feasible. Such comparisons can test whether the profiles identify failures and guide repairs that aggregate success scores leave unresolved.

\section*{Acknowledgments}

A colleague who wishes to remain anonymous pointed out in workshop that the bound above $\delta_0$ in Observation~\ref{obs:unif-discrete} is an inequality in general.

\newpage
\bibliography{references_new}

\newpage
\appendix

\section{Pilot audit}\label{sec:pilot}

We demonstrate the audit on a pretrained system. The materials are available in the \href{https://github.com/deltaquebec/grounding/tree/main}{project repository}.

\subsection{System and evaluation index}

The audited system is the classic word2vec \cite{mikolov2013distributed}, trained on Google News, and used here in the pruned 43{,}981-word, 300-dimensional distribution packaged with NLTK. The declared system is the released model, with its skip-gram training as history $\mathcal T$.

The atomic vocabulary $\Sigma_{\text{atom}}$ consists of the released words, and $\Phi$ is their learned vector lookup; its acquisition, therefore, counts as own learning under the declared history. For composites, the auditor extends $\Phi$ by adding the atomic vectors and normalizing the result to unit length. This extension and the evaluation readouts $\mathcal A$ are instrument components, outside the system's G0 certificate. The geometric comparisons use cosine similarity.

The audits share a linguistic context and inferential meaning type, with coordinate-specific populations, readouts, and discrepancy measures. For G1, each atom is assigned a profile of model-cosine quantiles against its rated partners. Target profile contains the corresponding gold-rating quantiles, and its discrepancy is their mean absolute difference.

For composite accuracy, the readout orders variants of a base phrase, and the target specifies their intended ordering. Discrepancy is measured by Kendall. The robustness audit instead measures angular change in the pooled phrase vectors, as specified below. These measurements use different interfaces, and so must be read with their respective specifications.

The atomic data are MEN pairs \cite{bruni2014multimodal} and SimLex-999 pairs \cite{hill2015simlex}, providing standards of relatedness and similarity, respectively. The composite population is an adjective-noun fragment drawn from $1{,}065$ adjectives and $20$ nouns, with synonym and antonym tiers curated from WordNet \cite{fellbaum1998wordnet}. Human phrase-similarity judgments \cite{mitchell2010composition} would certainly provide an alternative standard, requiring a corresponding scoring specification.

The preservation tolerance was $\bar\varepsilon_{\text{pres}}=0.20$ on the quantile scale; the median atomic discrepancy and the fraction of atoms within tolerance are reported separately. Composite item accuracy used Kendall tolerance $\tau=0.25$, and the causal-effect threshold was $\eta=0.10$. An item tolerance and an acceptance bound for aggregate composite discrepancy are separate declarations. Robustness and systematicity bounds $\bar\omega$ and $\bar\beta$ were undeclared, so those coordinates remain descriptive. The random seed $0$.

For the principal dispositional test, $M$ is the learned word-to-vector assignment; its off-condition randomly permutes that assignment. Spurious covariates are undeclared, so $P^{\text{do}}=P$ for the relevant task population. A second intervention projects out the ten leading principal components of the MEN vocabulary. Both tests concern present contribution. A proposed historical control would retrain on corpora with co-occurrence statistics decoupled from the rating distribution; that control remains unexecuted, and historical selection remains uncertified.

Robustness is summarized by sampled perturbation medians and $90$th percentiles, corresponding to $\alpha=0.50$ and $\alpha=0.10$, respectively; these are quantiles of sampled effects, distinct from certificates of the per-item supremum defining $\omega_U$. Table~\ref{tab:pilot} shows both quantiles for the typo protocols and medians for synonym substitutions.

Under the principal typo protocol, each edit selects a token independently, allowing successive edits to affect different tokens. Under the secondary protocol, all edits affect one token. An edited token outside the vocabulary is dropped. If no tokens remain, the phrase is excluded and the exclusion rate is reported alongside the retained-item results.

The atomic audit divides rated pairs into disjoint calibration and evaluation halves. Atoms may recur across halves: the split estimates the quantile transforms on calibration pairs and applies them to evaluation pairs. The composite audit splits at the adjective level, with the twenty nouns shared by design.

The composite split tests the stipulated instrument on disjoint adjective sets: both the WordNet standard and the pooling rule are supplied by the auditor. The reported $\hat\beta$ is, therefore, accuracy on audit-held-out composites under that rule. Their exclusion from word2vec pretraining remains unverified, so the controlled training-withheld systematicity test belongs to the grid-world audit. Agreement across the adjective split describes stability across those lexemes; learned composition remains outside the pilot's scope.

\subsection{Metrics}

Let $F_{\operatorname{mod}}$ and $F_{\operatorname{gold}}$ be the empirical distribution functions of model cosine similarity and human ratings on the calibration half of the relevant dataset. The discrepancy for a rated pair is
\[
\left|F_{\operatorname{mod}}(\operatorname{cos})-F_{\operatorname{gold}}(\operatorname{rating})\right|.
\]
For each atom $\sigma$, $\varepsilon(\sigma)$ averages this discrepancy over evaluation pairs containing it. The calibration distributions fix the transforms. On each fixed set of rated partners, mean absolute difference is a metric on the resulting quantile profiles. Underlying vectors or ratings that produce identical profiles lie at distance zero.

Robustness uses angular distance between nonzero pooled vectors, normalized to $[0,1]$:
\[
d_{\operatorname{ang}}(x,y)=\frac{1}{\pi}\arccos\left(\frac{\langle x,y\rangle}{\lVert x\rVert\lVert y\rVert}\right).
\]
This is a metric on unit-normalized vectors. We draw one perturbation per item at each scale and reports the resulting empirical quantile $\hat q_U(\varepsilon)$; it measures effects under the sampling protocol, whereas $\omega_U$ takes the worst admissible perturbation at each item before taking a population quantile.

For a fixed set of retained items, each sampled effect is bounded above by that item's worst-case effect. Applying the same empirical quantile rule preserves this ordering. A lower bound on the population modulus must also account for how items were sampled and excluded: sampling introduces uncertainty, and exclusions change the evaluated population. Increasing the number of items while retaining one draw per item still estimates the sampled-effect distribution. Within-item search may improve a lower estimate of worst-case change; an upper certificate requires exhaustive coverage or a justified bound. Because perturbations are drawn separately at each scale, the empirical curve may decrease between scales.

For each base phrase, WordNet supplies synonym-substituted, antonym-substituted, and unrelated variants, with the stipulated similarity order
\[
\operatorname{syn}>\operatorname{ant}>\operatorname{unrelated}.
\]
For strict predicted orderings, Kendall is the fraction of the three variant pairs ordered discordantly, taking values in $\{0,\frac13,\frac23,1\}$. At item tolerance $\tau=0.25$, only perfect ordering passes.

\subsection{Results}

\begin{table}
    \centering
    \caption{Profile of pruned word2vec. Populations are identified by row. Thresholds used for comparisons were fixed before inspection; undeclared acceptance bounds leave measurements descriptive. G3 reports sampled perturbation quantiles.}
    \label{tab:pilot}
    \small
    \begin{tabular}{
        l
        >{\raggedright\arraybackslash}p{3.4cm}
        >{\raggedright\arraybackslash}p{5.2cm}
        >{\raggedright\arraybackslash}p{4.2cm}}
        \toprule
        \textbf{Clause}
        & \textbf{Instrument}
        & \textbf{Quantity}
        & \textbf{Value} \\
        \midrule
        \textbf{G0}
        & inspection
        & atomic $\Phi$: system, own learning
        & strong \\
        &
        & composite extension; readout $\mathcal A$
        & instrument \\

        \textbf{G1}
        & MEN, evaluation pairs
        & median atomic discrepancy; fraction within $0.20$
        & $0.138$; $77.9\%$ \\
        & SimLex, evaluation pairs
        & median atomic discrepancy; fraction within $0.20$
        & $0.186$; $54.9\%$ \\

        \textbf{G2a}
        & WordNet tiers, all adjectives
        & mean Kendall discrepancy; perfect-order rate
        & $0.193$; $42.1\%$ \\

        \textbf{G2b}
        & permutation ablation, MEN
        & $\widehat{\text{ACE}}$ [bootstrap $95\%$ interval];
          $\eta=0.10$
        & \makecell[l]{
          $0.292$ $[0.283,0.301]$; \\
          estimate above $\eta$; \\
          history uncertified} \\
        & principal-subspace ablation, MEN
        & $\widehat{\text{ACE}}$ [bootstrap $95\%$ interval];
          ten components
        & \makecell[l]{
          $0.044$ $[0.036,0.052]$; \\
          estimate below $\eta$} \\

        \textbf{G3}
        & typo, $\varepsilon=1,2,3$;
          token selected per edit
        & \makecell[l]{
          $\hat q_U$: median; \\
          $90$th percentile; \\
          excluded fraction}
        & \makecell[l]{
          $0.241$, $0.243$, $0.245$; \\
          $0.255$, $0.318$, $0.463$; \\
          $0$, $0.39$, $0.62$} \\
        & typo, $\varepsilon=1,2,3$;
          one token throughout
        & \makecell[l]{
          $\hat q_U$: median; \\
          $90$th percentile}
        & \makecell[l]{
          $0.241$, $0.242$, $0.242$; \\
          $0.255$ throughout} \\
        & synonym, $\varepsilon=1,2$
        & $\hat q_U$: median
        & $0.267$, $0.376$ \\

        \textbf{G4}
        & audit-held-out adjectives
        & mean Kendall discrepancy;
          $\hat\beta$ at $\tau=0.25$, $\bar\beta$ undeclared
        & $0.197$; $41.0\%$ \\
        \bottomrule
    \end{tabular}
\end{table}

Table~\ref{tab:pilot} shows different agreement profiles under the two lexical standards. MEN relatedness gives lower median discrepancy and a higher within-tolerance fraction than SimLex similarity. Pairwise Spearman correlations are $0.77$ and $0.48$, respectively, consistent with reported performance for this model family \cite{bruni2014multimodal,hill2015simlex}. Both median discrepancies fall below $0.20$; the standards produce different error distributions, rather than opposite median-based verdicts.

The typo results depend strongly on the front end: $86.8\%$ of single-edit tokens leave the pruned vocabulary and are dropped. The median sampled effect is $0.241$ at the first edit and changes little thereafter. The protocols separate in the upper tail: the $90$th percentile stays at $0.255$ when all edits affect one token, but reaches $0.463$ when each edit selects a token independently.

Exclusions also increase under the latter protocol, reaching $62\%$ at three edits. The retained-item quantiles therefore concern changing populations. Assigning excluded phrases maximal deviation instead raises the three-edit median to $1.0$. Under synonym substitution, the median sampled effect increases from $0.267$ to $0.376$. These curves, shown in Figure~\ref{fig:modulus}, describe the sampling and exclusion protocols. Worst-case robustness depends on the per-item supremum.

\begin{figure}
    \centering
    \includegraphics[width=0.62\textwidth]{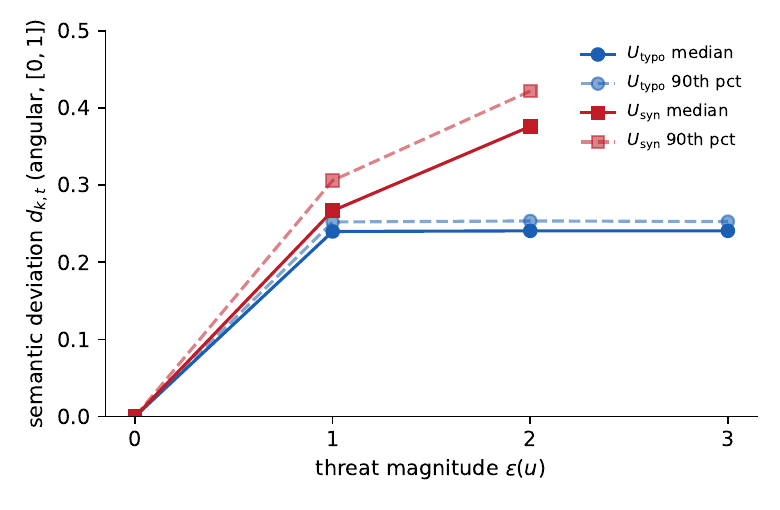}
    \caption{Sampled perturbation quantiles $\hat q_U$,
    using one draw per retained item at each scale.
    Typo results depend on the edit and exclusion policies.
    The curves are descriptive with respect to $\omega_U$.}
    \label{fig:modulus}
\end{figure}

Perfect composite tier ordering holds for $43.2\%$ of calibration phrases and $41.0\%$ of audit-held-out phrases. The similar rates suggest limited change across this adjective split, while fewer than half of either population meet the item criterion. Median angular change is smaller under antonym substitution ($0.252$) than under synonym substitution ($0.267$). This contrast is consistent with difficulty distinguishing the declared tiers. Individual ordering errors may arise in other ways.

Fitting a single adjective--noun pooling weight selects equal weighting, $0.5$, and leaves held-out accuracy unchanged. The composite results, therefore, concern learned atomic vectors under an auditor-supplied composition rule. Agreement with that same rule at the vector-pooling stage is stipulated.

For the principal causal test, $M$ is the learned word-to-vector assignment over the MEN vocabulary, including both halves. Random permutation preserves the vector collection while breaking its learned association with words. Using the same sample of comparisons between rated pairs in both conditions, human-concordant ordering falls from $0.787$ to $0.495$, giving
\[
\widehat{\text{ACE}}=0.292,\qquad95\%\text{ bootstrap interval }[0.283,0.301].
\]
The estimate exceeds $\eta=0.10$. The bootstrap treats comparisons as independent, although they share underlying MEN pairs; its interval may, therefore, understate uncertainty. A dependence-aware analysis is needed to validate nominal coverage. The intervention supports a present contribution of the learned assignment, with performance under this permutation close to chance.

Projecting out the ten leading principal components lowers concordance from $0.787$ to $0.743$, giving $\widehat{\text{ACE}}=0.044$ with reported interval $[0.036,0.052]$, below $\eta$. Substantial concordance does survive this intervention: measured contribution depends on the designated mechanism and off-condition.

Note, however, that historical selection remains uncertified. The proposed interpretation here requires both a model of retention under skip-gram optimization and evidence that the audited effect contributed to that retention. Training rewards co-occurrence prediction, whereas the audit scores concordance with human ratings. A common source in linguistic competence could well explain their association. We leave that aside for now.

As before, $P^{\text{do}}=P$. Permutation measures present contribution, and the historical question needs a different test. Retraining on corpora whose co-occurrence statistics are decoupled from the ratings is a candidate control, subject to the retention assumptions of Section~\ref{subsubsec:g2bselection}.

G0-strong applies to the atomic mapping under the declared training history. The composite extension and readouts remain instrument components. Accordingly, we establish measurements conditional on those instruments and evidence of present contribution.

\subsection{Scope and limitations}

The typo profile is specific to the pruned vocabulary and token-dropping policy; the corresponding out-of-vocabulary rate for the full release is unmeasured here. The composite standard is ordinal and lexicographer-curated. Graded human phrase judgments \cite{mitchell2010composition} would provide a complementary standard, with an adapted scoring protocol.

The adjective split withholds items from audit calibration, and their presence in pretraining is unknown. It differs from generalization tests under controlled training \cite{lake2018generalization,kim2020cogs}. The reported $\hat\beta$ remains descriptive, because $\bar\beta$ was undeclared. Historical selection also remains uncertified: the interventions measure present contribution, and historical controls would have to show whether concordance with human ratings contributed to retaining the mechanisms.

\subsection{Relation to quantitative grounding studies}\label{subsec:pilotlit}

Existing methods measure several quantities used by the profile. The framework organizes them by their targets, populations, instruments, and evidential scope. Their results become comparable only where those specifications support comparison, of course.

Probing recovers information from internal representations \cite{hewitt2019structural,belinkov2022probing}. Applications examine color \cite{abdou2021can}, spatial structure \cite{patel2022mapping}, entities and situations \cite{li2021implicit}, board states \cite{li2023emergent}, and space and time \cite{gurnee2024language}. A probe can supply a readout for G1 or G2a, depending on whether the audited items are atoms or composites.

Representation-alignment methods compare similarity structure. Representational similarity analysis \cite{kriegeskorte2008representational} and centered kernel alignment \cite{kornblith2019similarity} provide distinct comparison procedures, surveyed alongside related methods by \cite{sucholutsky2023getting}. Cross-model convergence \cite{huh2024platonic} and connections between alignment and grounding \cite{tjandrasuwita2025understanding} motivate such comparisons. Our quantile profiles draw on similarity structure, while supplying an explicit atomic readout and target.

Causal abstraction \cite{geiger2021causal}, mediation analysis, and interventions associated with model editing \cite{meng2022locating} can inform the dispositional component of G2b. Their relation to $\text{ACE}(M)$ depends on the mechanism, off-condition, success predicate, and intervention distribution. Historical certification requires separate evidence of selection for the audited effect.

The contribution at this level is organizational. A common reporting framework makes the dependencies and gaps among these measurements explicit. The pilot illustrates their combination, while retaining the limits of its stipulated instruments, sampled threats, and unresolved selection history.

\section{Grid-world audit}\label{sec:gridworld}

This audit examines a navigator trained on color-direction commands in a continuous plane. It compares atomic and composite accuracy, agreement with the declared additive rule, performance on withheld combinations, causal contribution, and perturbation sensitivity.

\subsection{System and specification}

The declared system is the trained network, with history $\mathcal T$ running from initialization to the audited weights. The encoder consists of an embedding layer and a GRU with representation space $\mathcal R=\mathbb R^{64}$. A learned linear readout $W:\mathbb R^{64}\to\mathbb R^2$ produces coordinates. All three components belong to the system and were acquired within its training history, supporting G0-strong.

The atomic vocabulary is
\[
\Sigma_{\text{atom}}=\{\text{RED},\text{BLUE},\text{NORTH},\text{SOUTH},\text{EAST},\text{WEST}\}.
\]

Colors denote landmarks,
\[
\mathcal I(\text{RED})=(8,8),\qquad\mathcal I(\text{BLUE})=(2,2),
\]
and directions denote their corresponding unit vectors. The semantic operation is addition; for example, $\mathcal I(\text{RED NORTH})=(8,9)$. The meaning metric is Euclidean distance. $P$ is uniform over the eight color-direction commands.

REINFORCE training used a Gaussian policy over displacements for $20{,}000$ episodes in batches of $16$, with seed $0$. Training included atoms and six composite commands; \texttt{BLUE EAST} and \texttt{RED WEST} were withheld. Reward combined proximity to the target with a bonus inside a $0.5$-unit radius. The audited success predicate is target discrepancy at most $\tau=0.5$.

G1 takes the maximum over six atoms; G2a takes the maximum over all eight composites, including the withheld pair. Compositional deviation is evaluated on the eight constructor applications, with trained and withheld maxima reported separately. Systematicity uses the two withheld commands. The audit $\eta=0.10$ for the causal effect, but leaves the acceptance bounds for the other measurements undeclared.

Discrepancies, compositional deviations, success rates, and ablation effects are evaluated exhaustively for the audited readouts. The hidden-state supremum is analytic; directional averages and embedding linearizations have the more limited status specified below. Coordinates and discrepancies are rounded separately.

\subsection{Atomic and composite accuracy}

Every atomic discrepancy is below $0.1$. For \texttt{RED},
\[
g(\text{RED})=(8.034,7.955), \qquad \varepsilon(\text{RED})=0.056.
\]
The largest atomic discrepancy is $0.098$, at \texttt{WEST}, realized at $(-1.077,-0.062)$.

Every trained composite has discrepancy below $0.08$. For example,
\[
g(\text{RED NORTH})=(7.994,8.938), \qquad \varepsilon(\text{RED NORTH})=0.062
\]
against target $(8,9)$. The withheld outputs are
\[
g(\text{BLUE EAST})=(2.769,1.781), \qquad g(\text{RED WEST})=(7.518,8.260),
\]
with discrepancies $0.318$ and $0.580$, respectively. Thus, the maximum composite discrepancy is $0.580$, the success rate over $P$ is $7/8$, and $\beta=1/2$ on the withheld population.

\subsection{Interventions and historical status}
The designated mechanism $M$ is the GRU's sequential integration of tokens beyond the first. Off-condition restricts the encoder to the first token. For \texttt{RED NORTH}, the output then becomes $(8.034,7.955)$, the realized meaning of the atom \texttt{RED}.

All seven previously successful commands fail under ablation; \texttt{RED WEST} fails in both conditions. As in Section~\ref{sec:toy}, $P^{\text{do}}=P$, giving
\[
\text{ACE}(M)=0.875\geq\eta=0.10.
\]
The effect on the six trained commands is $1.0$. Mean target distance increases by $0.853$ units, a separate continuous measure of the intervention's effect.

Zeroing the direction embedding while retaining the GRU step also gives a binary success effect of $0.875$. The agreement shows that both specified interventions eliminate the observed successes. Under both, directional content remains confounded with the other processing they disrupt.

Historical selection remains uncertified. The proposed interpretation is that success within the radius contributed to retaining $M$. Retraining with a direction-independent reward, computed against the landmark, would test the role of direction-dependent reward. That control remains unexecuted. Separating retention for the success bonus from retention for proximity would require a further comparison.

\subsection{Robustness}

The hidden-state threat includes all additive perturbations with $\lVert\Delta r\rVert_2\leq\varepsilon$. Since $\lVert W\rVert_{\operatorname{op}}=1.680$,
\[
\sup_{\lVert\Delta r\rVert_2\leq\varepsilon} \lVert W(r+\Delta r)-Wr\rVert_2 = 1.680\,\varepsilon.
\]
The value is independent of $r$, so every population quantile equals this supremum. At radius $0.5$, $\omega_U(0.5)=0.840$.

For comparison, $256$ standard normal vectors normalized to radius $0.5$ give mean output displacement $0.095$. For $h$ uniformly distributed on that sphere,
\[
\sqrt{\mathbb E\lVert Wh\rVert_2^2} = \frac{0.5\,\lVert W\rVert_F}{\sqrt{64}} \approx0.110, \qquad \lVert W\rVert_F=1.767.
\]
This population root mean square bounds the population mean; the sampled mean describes directional averages, which lie at or below the supremum required by G3.

For token-embedding perturbations, let $J_\sigma$ be the Jacobian of the GRU followed by the readout at the clean embeddings of command $\sigma$. The reported linearized value at radius $0.5$ is the $0.90$-quantile over the eight commands of $0.5\lVert J_\sigma\rVert_{\operatorname{op}}$, namely $0.240$. The operator norm depends on the embedding perturbation norm and the Euclidean output norm. We evaluate scales $0.1$, $0.25$, $0.5$, and $1.0$.

This calculation gives the worst direction in each local linear approximation. A finite-radius bound for the nonlinear map additionally requires control of the remainder. Against the illustrative bound $\bar\omega(\varepsilon)=\varepsilon$, the hidden-state supremum would fail. The embedding estimate lies below the bound.

\subsection{Composition}

For \texttt{RED NORTH}, the realized atoms give
\[
g(\text{RED})+g(\text{NORTH})=(8.034,7.955)+(-0.033,0.992)=(8.001,8.947).
\]
The distance from the realized composite $(7.994,8.938)$ is $0.011$. The maximum deviation on trained pairs is $0.062$, compared with withheld deviations of $0.317$ for \texttt{BLUE EAST} and $0.670$ for \texttt{RED WEST}.

For \texttt{RED WEST}, the sum of realized atoms is $(6.957,7.893)$, at distance $0.115$ from target $(7,8)$. The realized composite lies $0.670$ from that sum and $0.580$ from the target. Thus, the atomic errors would permit success under the declared additive rule, while the realized combination crosses the success boundary in Figure~\ref{fig:redwest}.

\begin{figure}
    \centering
    \begin{tikzpicture}[x=3.4cm, y=3.4cm, >=Stealth]
        \draw[gray!60, ->] (6.35,7.55) -- (8.15,7.55)
            node[below, font=\scriptsize, text=black] {$x$};
        \draw[gray!60, ->] (6.35,7.55) -- (6.35,8.65)
            node[below right, font=\scriptsize, text=black] {$y$};
        \foreach \x in {6.5,7,7.5,8} {
            \draw[gray!60] (\x,7.55) -- (\x,7.53)
                node[below, font=\scriptsize, text=black] {\x};
        }
        \foreach \y in {7.6,8,8.4} {
            \draw[gray!60] (6.35,\y) -- (6.33,\y)
                node[left, font=\scriptsize, text=black] {\y};
        }
        \draw[black, dashed] (7,8) circle (0.5);
        \node[font=\scriptsize, anchor=south west]
            at (7.36,8.36) {success radius $0.5$};
        \draw[black, dotted, thick] (6.957,7.893) -- (7,8)
            node[midway, above left, font=\scriptsize,
                 xshift=-1pt, yshift=1pt] {$0.115$};
        \draw[black, dashed, thick]
            (6.957,7.893) -- (7.518,8.260)
            node[pos=0.55, below right, font=\scriptsize] {$0.670$};
        \draw[black!70] (7,8) -- (7.518,8.260)
            node[pos=0.5, above left, font=\scriptsize] {$0.580$};
        \fill[black] (7,8) circle (1.6pt)
            node[below right, font=\scriptsize,
                 xshift=3pt, yshift=-3pt] {target $(7,8)$};
        \draw[black, thick]
            (6.957-0.022,7.893-0.022)
            rectangle (6.957+0.022,7.893+0.022);
        \node[font=\scriptsize, anchor=north east]
            at (6.94,7.86) {sum of realized atoms};
        \fill[black] (7.518,8.260)
            ++(0,0.03) -- ++(-0.026,-0.045)
            -- ++(0.052,0) -- cycle;
        \node[font=\scriptsize, anchor=west]
            at (7.55,8.27) {realized composite};
    \end{tikzpicture}
    \caption{\texttt{RED WEST} in the grid world.
    The sum of the realized atoms (open square) lies
    $0.115$ from the target (filled circle), inside the
    success radius. The realized composite (filled triangle)
    lies $0.670$ from that sum and $0.580$ from the target,
    outside the radius.}
    \label{fig:redwest}
\end{figure}

These discrepancies suggest compositional training or an architectural constraint as candidate interventions. Controlling the readout norm could likewise reduce hidden-state sensitivity. Both repairs remain untested; evaluating retrained composition would require a fresh withheld split. Changing the item tolerance changes the assessment of the same system: radius $0.6$ accepts all eight commands, while radius $0.1$ rejects both withheld commands.

\subsection{Supervised variant}
A supervised variant uses the same architecture and withheld combinations, trained by gradient descent on squared coordinate error for $3{,}000$ episodes at seed $0$. Atomic discrepancies are below $0.009$, and trained composite discrepancies are below $0.06$. The withheld errors are $0.458$ and $0.337$, both within the success radius, giving $\beta=1$ at $\tau=0.5$.

The withheld compositional deviations remain $0.458$ and $0.334$, compared with a trained maximum of $0.057$. Improved withheld accuracy, therefore, coexists with larger departures from addition on withheld than on trained pairs. This comparison separates accuracy from agreement with the semantic algebra.

\section{Format definitions}\label{sec:modedefs}

We specify the representation formats and referential coupling compared in Table~\ref{tab:modes}. Formats describe the organization of representations; coupling describes their causal connection to the environment. They may coexist, and any of them can accompany any grounding verdict. The suggested contexts, meaning types, and threats are simply starting points for an audit, whose populations, metrics, aggregates, and acceptance bounds must still be declared.

Provenance is assessed component by component. Designed or imported system components support at most G0-weak; components acquired through the system's own learning or evolution may support G0-strong. An auditor-supplied readout or composition mechanism remains part of the instrument. These records assign credit for mechanisms. Original intentionality and the zero semantical commitment condition belong to theories of content. In every format, G2b separately requires present causal contribution and historical selection for the audited effect.

\subsection{Symbolic format}
Symbolic representations organize discrete expressions through explicit tokens, constructors, definitions, or rules. Logical knowledge bases provide one realization \cite{levesque2001logic}. A simple specification takes
\[
\mathcal R=\Sigma,\qquad\Phi=\operatorname{id}_{\Sigma},
\]
with an edit metric and a readout implemented by lookup or rule-based interpretation. An inferential audit in a logical context may use well-formed expressions as its population and declared surface edits as threats.

Agreement with an analyst's interpretation depends on the implementation: if, in a particular implementation, $g=\mathcal I$ on atoms and both maps extend through the same typed semantic algebra, then atomic and composite discrepancies vanish throughout the interpreted closure. Compositional deviation is also zero, and every valid withheld composite passes at any nonnegative item tolerance (see Theorem~\ref{thm:hom}; Corollary~\ref{cor:coincide}). Expressions outside the declared grammar belong to a different test; they require an extended grammar and target specification.

For a uniformly discrete representation space with minimum spacing $s>0$, admissible perturbations below scale $s$ act as the identity on representations when the declared perturbation scale dominates the representation distance. At the first admissible edit, semantic change depends on the readout and threat family. G3 concerns the size of that change, whatever error correction the system performs. Malformed inputs require an explicit evaluation policy. A designed interpreter may satisfy G0-weak. A learned interpreter may satisfy G0-strong, subject to the provenance of every component credited to it. Each G2b status can occur under either verdict.

\subsection{Vectorial format}

Vectorial representations are elements of a vector space, typically taken to be $\mathcal R=\mathbb R^n$, equipped with a declared metric and an encoder $\Phi:\Sigma\to\mathcal R$. The encoder may be learned, designed, or imported. The format permits geometric accounts of semantic structure \cite{mollo2023vector}. High dimensionality and learning are optional.

An audit here concerns inferential meanings in linguistic tasks or extensional meanings in vision-language tasks. Image-text contrastive training, as in CLIP \cite{Radford2021LearningTV}, supplies a learned alignment whose accuracy, provenance, and causal role require separate assessment. Acquisition status attaches to components and their histories.

The optional geometric measurement compares representational distances with distances between target meanings. Both metrics and their scales must, of course, be declared. Accuracy under a particular readout calls for a readout-level comparison. The pilot's G1 instrument instead compares model and gold quantile profiles for each atom; its status as an accuracy measurement depends on that explicit readout and target construction.

Analogy patterns and changes in geometry under fine-tuning \cite{rajaee2021doesfinetuningaffectgeometry} likewise concern representational organization. For G2b, an image-text mechanism must pass the declared intervention test and have evidence of retention for the audited effect.

It should be emphasized that robustness under continuous coordinates is an empirical matter. Results concerning semantic perturbations or model regularity apply under their particular threats and assumptions \cite{dhole2025adversem,singh2024robustness,pmlrv202catellier23a}. For a linear readout under norm-bounded perturbations, the operator norm supplies a global bound; for nonlinear readouts, a local Jacobian requires remainder control before it yields a finite-radius certificate.

Vector arithmetic in static embeddings \cite{mikolov2013efficient,penningtonetal2014glove} is likewise consistent with composition supplied by the auditor. If the auditor supplies pooling, that mechanism is recorded as instrument. Contextual models require system- and construction-specific tests \cite{xu2024largelanguagemodelscompositional,dziri2023faithfatelimitstransformers}, with compositional deviation distinguished from target accuracy on withheld combinations.

\subsection{Relational format}

Relational representations organize typed objects and relations, such as nodes, edges, roles, subsumptions, or axioms. A realization may use a first-order structure, a graph, or a description-logic knowledge base, together with a declared edit or graph metric. The format can overlap with symbolic or vectorial implementations.

For an inferential audit, the readout may return consequences or role descriptions under a specified logic. A stipulated lexical analysis might associate the classic \textit{bachelor} with the description
\[
\textsc{human}\sqcap\textsc{male}\sqcap \neg\textsc{married}.
\]
This specifies a candidate interpretation. Identifying inferential closure with Fregean sense would be a substantive thesis. Likewise, a description of \textit{dragon} can support inferential assessment while its reference stays unsettled.

Using the same role language for outputs and targets makes comparison possible. Agreement then has to be measured. G1 compares atomic descriptions, while G2a compares composite descriptions under the declared discrepancy measure. The audit must specify how missing information, negation, and inconsistency affect the readout.

A relational-edit threat declares which nodes, edges, types, or axioms may change, their edit costs, and the evaluated population. Changes to relations can propagate through inference \cite{silva2019exploring}; that propagation can be unbounded even in sparse, typed graphs. Edits that change the intended content require adjusted targets rather than an invariance test.

Compositional deviation is zero only when the realized operations respect the declared semantic algebra. Withheld-composite performance remains a separate measurement. Learned relational components \cite{shengyuan2023differentiable,bosselut2021dynamic} receive provenance records alongside any designed ontology or inference engine. For G2b, the audit designates a mechanism and tests its present contribution and retention for a predicate such as entailment or question-answering success. Missing historical documentation leaves selection uncertified.

\subsection{Referential coupling}

Referential coupling is a sensorimotor connection between representations and the environment through perception and action. It can occur in any representation format. An embodied extensional audit may take $\mathcal R$ to contain sensorimotor encodings and $\mathcal M^{\text{ext}}$ to represent relevant world or task states. A coupled system can refer inaccurately, and its readout can have any acquisition and selection history.

G1 and G2a compare atomic and composite outputs with world-involving targets. For suitable set representations, $1-\operatorname{IoU}$ provides a metric; combining it with relational penalties requires justification of the resulting measure. A shift in the reference distribution can change aggregate discrepancy in either direction and must be recorded as a change in the audit conditions.

For G2b, the designated mechanism might be a visual or spatial-relation module. Certification requires both $\text{ACE}(M)\geq\eta$ under the declared intervention and evidence that the audited effect contributed to retaining $M$. Alternative designs must state what they identify.

Candidate threats include lighting, viewpoint, pose, and sensor noise, restricted or interpreted according to whether the relevant target is preserved. Novel attribute combinations and spatial relations require separate measurements of deviation from the declared algebra and accuracy against withheld targets.

Direct sensorimotor protocols are most straightforward for perceptually accessible task states. Theoretical entities may nevertheless support extensional assessment through indirect measurement; abstractions and fictional kinds may require inferential, social-normative, or discourse-relative standards. These differences concern the justification of the audit interface.

\section{Grounding under change of representation}\label{sec:translation}

A change of representation can preserve some profile coordinates while altering others. Let architectures $\mathfrak G$ and $\mathfrak G'$ share the surface domain, meaning space $\mathcal M$, and standard $(\mathcal I,d)$, where $d$ is a metric or pseudometric. Write
\[
g=\mathcal A\circ\Phi, \qquad g'=\mathcal A'\circ\Phi',\qquad h:\mathcal R\to\mathcal R'
\]
for their realized interpretations and the proposed translation. Comparisons below use the same item populations and aggregation rules.

Structural translations take several forms. Tensor representations encode symbolic structure \cite{smolensky1990tensor}, while compositional distributional semantics specifies how grammatical operations act on vector representations \cite{coecke2010mathematical,maillard2014type}. Geometric realizations of minimalist grammars \cite{stabler1996derivational,beimgraben2012geometric} and vector-space reconstructions of formal semantics \cite{quigley2025vector} likewise connect formal operations with representational implementations. The audit checks for which measured properties these connections preserve, under explicit error and provenance conditions.

\subsection{Accuracy}
Approximate agreement between encoders must be distinguished from approximate agreement between readouts. Suppose, on the evaluated population,
\[
d_{\mathcal R'}\bigl(\Phi'(\sigma),h(\Phi(\sigma))\bigr) \leq a, \qquad d\bigl(\mathcal A'(h(\Phi(\sigma))), \mathcal A(\Phi(\sigma))\bigr)\leq b.
\]

If $\mathcal A'$ has an oscillation bound $\nu$ on the relevant representation pairs, then the triangle inequality gives
\[
d\bigl(g'(\sigma),g(\sigma)\bigr)\leq \nu(a)+b=:\zeta.
\]
For an $L$-Lipschitz readout, $\nu(a)=La$. Without such control, small encoder error may coincide with large semantic error. A direct bound on $d(g'(\sigma),g(\sigma))$ can instead be supplied.

Let $\varepsilon'(\sigma)=d(g'(\sigma),\mathcal I(\sigma))$. We see
\[
\bigl|\varepsilon'(\sigma)-\varepsilon(\sigma)\bigr| \leq\zeta.
\]
A uniform bound, therefore, transports maxima, means, and fixed quantiles over the same weighted population to within $\zeta$. This applies separately to G1 and G2a. Passing a declared bound is preserved only when the source measurement has sufficient margin below it.

Systematicity additionally depends on threshold crossings. On the same withheld population, define
\[
B(q):= \frac{1}{|\mathcal D_{\text{novel}}|} \sum_{\sigma\in\mathcal D_{\text{novel}}} \mathbf 1[\varepsilon(\sigma)\leq q],
\]
with $B(q)=0$ for $q<0$. At item tolerance $\tau$,
\[
B(\tau-\zeta)\leq\beta'\leq B(\tau+\zeta).
\]
Thus, small semantic translation error may change $\beta$ when many items lie near the success boundary.

\subsection{Composition}
Assume both audits use the same typed grammar and semantic operations. For a constructor $f$, suppose its semantic operation satisfies
\[
d\bigl(f^{\mathcal M}(x_1,\dots,x_n), f^{\mathcal M}(y_1,\dots,y_n)\bigr) \leq \sum_{i=1}^{n}L_i\,d(x_i,y_i)
\]
on the evaluated meanings. Set $e(s):=d(g'(s),g(s))$. Applying the triangle inequality to the two composition comparisons gives
\[
\left| \delta'_{\text{comp}}(f(\vec s)) - \delta_{\text{comp}}(f(\vec s)) \right| \leq e(f(\vec s))+\sum_{i=1}^{n}L_i e(s_i).
\]
If all these semantic translation errors are at most $\zeta$, the bound becomes $(1+\sum_i L_i)\zeta$. Aggregate transport requires a uniform bound over the declared constructor applications.

Exact agreement of realized meanings on the inputs and outputs preserves compositional deviation for any semantic operation. Approximate agreement requires control of how the semantic operation propagates errors. If the two audits use different semantic operations, their discrepancy must also be in the bound.

\subsection{Robustness}

Clean-item agreement constrains behavior only at the clean items. Transporting G3 requires correspondence between threats and between their reference populations.

One sufficient condition is as follows. Let $P'=h_{\#}P$ be the pushforward representation distribution, and suppose
\[
d\bigl(\mathcal A'(h(r)),\mathcal A(r)\bigr)\leq\zeta
\]
on the clean states and all source states reached by the threats under consideration. Assume that every target perturbation $v$ of scale at most $\varepsilon$, at every relevant $h(r)$, has a source counterpart $u$ satisfying
\[
v(h(r))=h(u(r)), \qquad \varepsilon(u)\leq\rho(\varepsilon),
\]
where $\rho$ is a declared nondecreasing scale conversion. Then, pointwise,
\[
D'_{U',\varepsilon}(h(r)) \leq D_{U,\rho(\varepsilon)}(r)+2\zeta.
\]
Taking quantiles at the same coverage gives
\[
\omega'_{U'}(\varepsilon) \leq \omega_U(\rho(\varepsilon))+2\zeta.
\]
This is a one-sided result. A reverse comparison requires the corresponding reverse threat coverage.

Metric distortion matters when scales are defined by representation distance. Target perturbations leaving the image of $h$, for example, fall outside the displayed assumption. Approximate correspondence requires an additional bound on its semantic error.

Consequently, two realizations can agree exactly on all clean expressions while having different robustness profiles. Uniform discreteness makes the modulus zero below the minimum spacing when scale controls distance; above that spacing, it constrains the first edit's effect only through the image diameter. A geometric realization may admit intermediate perturbations, and its modulus may still be large.

\subsection{Provenance and etiology}
If an auditor applies $h$ solely for measurement, then it is part of the instrument, and the source system's acquisition record is unchanged; if externally produced components are installed in a target system, then they are recorded as imported or designed, as appropriate, and support at most G0-weak for those components. The timing of translation is immaterial to this status.

When the student's own optimization fixes its parameters, the resulting components count as learned within that history, with teacher outputs recorded as a source of training information. G0-strong remains available subject to the provenance of all components credited to the student.

In both cases G2b remains open: matching a teacher's behavior, and training for imitation, are compatible with either answer to whether the student's designated mechanism was retained for the audited effect. Certification requires a retention argument for the student's own history and a present intervention on its mechanism.

A comparative experiment could fix a grammar and realize it symbolically and geometrically \cite{beimgraben2012geometric}. Both systems would be evaluated against shared targets, populations, and accuracy tolerances, with explicitly related representation-level threats or a common surface threat. The comparison would test whether closely matched accuracy and composition coexist with different robustness profiles. Geometry is one factor in the outcome, and the causal and historical coordinates require separate assessment for each realization.

\section{Formal results}\label{sec:thms}

Theorem~\ref{thm:modulus} characterizes uniform continuity of the readout $\mathcal A$ by a global modulus vanishing at zero; Lemma~\ref{lem:minimal} identifies its least such modulus, $\omega^\ast$. Corollary~\ref{cor:domination} bounds the threat-restricted quantile modulus $\omega_U$ by $\omega^\ast$ when the declared perturbation scale dominates representation distance. These results concern bounds on semantic change. Observation~\ref{obs:unif-discrete} gives the corresponding result for uniformly discrete domains.

For composition, Theorem~\ref{thm:hom} identifies pointwise zero deviation with the homomorphism condition. Corollary~\ref{cor:coincide} adds exact atomic preservation and recovers the intended homomorphic interpretation throughout the syntactic closure. Together, these results describe the exact structural limit of the graded profile.

\begin{theorem}[Modulus and uniform continuity]\label{thm:modulus}
    Fix $(k,t)$ and write $\mathcal A:(\mathcal R,d_{\mathcal R})\to(\mathcal M,d)$ for the readout. Then $\mathcal A$ is uniformly continuous if and only if there exists a global modulus $\omega:[0,\infty)\to[0,\infty]$, monotone nondecreasing with $\lim_{\varepsilon\downarrow0}\omega(\varepsilon)=0$, such that
    \[
    d\bigl(\mathcal A(r),\mathcal A(r')\bigr) \leq \omega\bigl(d_{\mathcal R}(r,r')\bigr) \qquad\forall\,r,r'\in\mathcal R.
    \]
    When these conditions hold, the oscillation function
    \[
    \omega^\ast(\varepsilon):= \sup\bigl\{ d(\mathcal A(r),\mathcal A(r')): d_{\mathcal R}(r,r')\leq\varepsilon \bigr\}
    \]
    is such a modulus.
\end{theorem}

\begin{proof}
    Suppose such a modulus exists. Given $\gamma>0$, choose $\delta>0$ with $\omega(\delta)<\gamma$. If $d_{\mathcal R}(r,r')<\delta$, monotonicity gives
    \[
    d\bigl(\mathcal A(r),\mathcal A(r')\bigr) \leq \omega\bigl(d_{\mathcal R}(r,r')\bigr) \leq\omega(\delta)<\gamma.
    \]
    Hence $\mathcal A$ is uniformly continuous.

    Conversely, suppose $\mathcal A$ is uniformly continuous and define $\omega^\ast$ as above. Its defining sets increase with $\varepsilon$, so $\omega^\ast$ is nondecreasing. Each pair $(r,r')$ belongs to the set at $\varepsilon=d_{\mathcal R}(r,r')$, giving
    \[
    d\bigl(\mathcal A(r),\mathcal A(r')\bigr) \leq \omega^\ast\bigl(d_{\mathcal R}(r,r')\bigr).
    \]

    Uniform continuity also sends zero-distance pairs to zero-distance pairs. Thus, $\omega^\ast(0)=0$. This observation covers pseudometric domains, where distinct points may have zero distance; on a metric domain, such a pair consists of the same point.

    Finally, given $\gamma>0$, choose $\delta>0$ such that $d_{\mathcal R}(r,r')<\delta$ implies $d(\mathcal A(r),\mathcal A(r'))<\gamma/2$. For every $0\leq\varepsilon<\delta$, all pairs in the defining supremum satisfy this inequality, so $\omega^\ast(\varepsilon)\leq\gamma/2<\gamma$. Therefore $\lim_{\varepsilon\downarrow0}\omega^\ast(\varepsilon)=0$, and $\omega^\ast$ is a modulus.
\end{proof}

\begin{lemma}[Minimality of $\omega^\ast$]\label{lem:minimal}
    If $\omega$ is any modulus satisfying Theorem~\ref{thm:modulus}, then
    \[
    \omega^\ast(\varepsilon)\leq\omega(\varepsilon) \qquad\forall\,\varepsilon\geq0.
    \]
\end{lemma}

\begin{proof}
    For any pair with $d_{\mathcal R}(r,r')\leq\varepsilon$,
    \[
    d\bigl(\mathcal A(r),\mathcal A(r')\bigr) \leq \omega\bigl(d_{\mathcal R}(r,r')\bigr) \leq\omega(\varepsilon).
    \]
    Taking the supremum over these pairs proves the claim.
\end{proof}

\begin{corollary}[Domination]\label{cor:domination}
    Let $U$ be a threat family acting on $\mathcal R$, with declared scales satisfying
    \[
    d_{\mathcal R}(r,u(r))\leq\varepsilon(u)
    \]
    for every $u\in U$ and every $r$ in the support of $P$. Let $\omega^\ast$ be the oscillation function defined in Theorem~\ref{thm:modulus}. Then
    \[
    \omega_U(\varepsilon)\leq\omega^\ast(\varepsilon)
    \]
    for every $\varepsilon\geq0$ and every coverage level $1-\alpha$, with $0\leq\alpha<1$.
\end{corollary}

\begin{proof}
    Fix $\varepsilon$ and $r$ in the support of $P$. Every $u$ with $\varepsilon(u)\leq\varepsilon$ satisfies $d_{\mathcal R}(r,u(r))\leq\varepsilon$, and hence
    \[
    d\bigl(\mathcal A(r),\mathcal A(u(r))\bigr) \leq\omega^\ast(\varepsilon).
    \]
    Taking the supremum over admissible $u$ gives $D_{U,\varepsilon}(r)\leq\omega^\ast(\varepsilon)$. If $\omega^\ast(\varepsilon)$ is finite, it is therefore an admissible bound at every coverage level. If it is infinite, the claimed inequality is immediate.
\end{proof}

\begin{observation}[Uniformly discrete domain]\label{obs:unif-discrete}
    Assume $(\mathcal R,d_{\mathcal R})$ is uniformly discrete: there exists $\delta_0>0$ such that $d_{\mathcal R}(r,r')\geq\delta_0$ whenever $r\neq r'$. Then every readout $\mathcal A:\mathcal R\to\mathcal M$ is uniformly continuous, and its oscillation function satisfies
    \[
    \omega^\ast(\varepsilon)=0 \quad(0\leq\varepsilon<\delta_0), \qquad \omega^\ast(\varepsilon) \leq\operatorname{diam}(\mathcal A(\mathcal R)) \quad(\varepsilon\geq\delta_0).
    \]
    If the domain has finite diameter
    \[
    D_{\mathcal R}:=\sup_{r,r'}d_{\mathcal R}(r,r')<\infty,
    \]
    then
    \[
    \omega^\ast(\varepsilon) =\operatorname{diam}(\mathcal A(\mathcal R))  \qquad(\varepsilon\geq D_{\mathcal R}).
    \]
    If the image diameter is finite, uniform discreteness also gives a global Lipschitz bound: for $r\neq r'$,
    \[
    d\bigl(\mathcal A(r),\mathcal A(r')\bigr) \leq\operatorname{diam}(\mathcal A(\mathcal R)) \leq \frac{\operatorname{diam}(\mathcal A(\mathcal R))} {\delta_0}\,d_{\mathcal R}(r,r').
    \]
    Thus $\mathcal A$ is $L$-Lipschitz with $L=\operatorname{diam}(\mathcal A(\mathcal R))/\delta_0$. An infinite image diameter leaves this bound uninformative; without further assumptions, the readout may fail to be Lipschitz.
\end{observation}

\begin{proof}
    If $d_{\mathcal R}(r,r')<\delta_0$, then $r=r'$. Taking $\delta=\delta_0$ therefore establishes uniform continuity and gives $\omega^\ast(\varepsilon)=0$ below $\delta_0$.

    Every image distance is bounded by the image diameter. If $\varepsilon\geq D_{\mathcal R}$, every pair in $\mathcal R$ enters the defining supremum, so the bound becomes an equality.

    The Lipschitz estimate follows from $d_{\mathcal R}(r,r')\geq\delta_0$ for distinct points, as displayed above.
\end{proof}

\begin{remark}
    For the discrete metric, $d_{\mathcal R}(r,r')=1$ whenever $r\neq r'$,
    \[
    \omega^\ast(\varepsilon)=0 \quad(\varepsilon<1), \qquad \omega^\ast(\varepsilon) =\operatorname{diam}(\mathcal A(\mathcal R)) \quad(\varepsilon\geq1).
    \]
    For a general uniformly discrete metric, the diameter bound can be strict at the first nonzero scale.
\end{remark}

\begin{remark}
    Discreteness is strictly weaker than uniform discreteness, and a readout on a discrete domain may fail to admit a global modulus vanishing at zero. Let
    \[
    \mathcal R=\{1/n:n\in\mathbb N,\ n\geq1\} \subset\mathbb R
    \]
    with the Euclidean metric, and define $\mathcal A(1/n)=0$ for even $n$ and $\mathcal A(1/n)=1$ for odd $n$. Every point of $\mathcal R$ is isolated, but
    \[
    \left|\frac1{2n}-\frac1{2n+1}\right|\longrightarrow0, \qquad \left|\mathcal A\left(\frac1{2n}\right) -\mathcal A\left(\frac1{2n+1}\right)\right|=1.
    \]
    Thus $\mathcal A$ is continuous but not uniformly continuous. By Theorem~\ref{thm:modulus}, it has no global modulus vanishing at zero.
\end{remark}

We now turn to composition. Pointwise zero deviation says that the realized interpretation respects every constructor. Exact atomic preservation then fixes which homomorphism it realizes.

\begin{theorem}[Exact homomorphism under pointwise zero deviation]\label{thm:hom}
    Fix $(k,t)$ and assume $d$ is a metric on $\mathcal M$. Let $\mathcal G$ be a typed signature of constructors, with no equations beyond typing, and let $\mathsf{Terms}(\Sigma_{\text{atom}},\mathcal G)$ be its free term algebra over the atomic vocabulary. Assume $\delta_{\text{comp}}(f(\vec\sigma))=0$ at every well-typed constructor application in the closure of $\Sigma_{\text{atom}}$ under $\mathcal G$.

    For a term $s$, write $g(s)$ as an abbreviation for $g(\text{surf}(s))$, where $\text{surf}(s)\in\Sigma$ is its surface realization. Then
    \[
    g: \big\langle \mathsf{Terms}(\Sigma_{\text{atom}},\mathcal G), \mathcal G \big\rangle \longrightarrow \big\langle\mathcal M,\mathcal F\big\rangle
    \]
    is a homomorphism: for every typed constructor $f$ and well-typed tuple $(s_1,\dots,s_n)$,
    \[
    g\bigl(f(s_1,\dots,s_n)\bigr) = f^{\mathcal M}\bigl(g(s_1),\dots,g(s_n)\bigr).
    \]
    For a pseudometric $d$, the conclusion holds in the metric quotient, provided every $f^{\mathcal M}$ sends coordinatewise zero-distance arguments to zero-distance values.
\end{theorem}

\begin{proof}
    The hypothesis states that
    \[
    d\Bigl( g\bigl(f(s_1,\dots,s_n)\bigr), f^{\mathcal M}\bigl(g(s_1),\dots,g(s_n)\bigr) \Bigr)=0
    \]
    at every well-typed constructor application. Since $d$ is a metric, the two values are equal. This is precisely the homomorphism condition. For a pseudometric, the same argument gives equality of equivalence classes, and the stated compatibility condition makes the quotient operations well-defined.
\end{proof}

A homomorphism from a free term algebra is uniquely determined by its values on atoms. Atomic preservation therefore yields the following consequence.

\begin{corollary}[Coincidence with intended semantics]\label{cor:coincide}
    In the setting of Theorem~\ref{thm:hom}, assume $\mathcal I$ is total on $\Sigma_{\text{atom}}$ and $\varepsilon(\sigma)=0$ at every atom. If $\mathcal I$ is partial there, restrict the term algebra to the atoms in its domain. Let
    \[
    \mathcal I^\uparrow: \mathsf{Terms}(\Sigma_{\text{atom}},\mathcal G) \to\mathcal M
    \]
    be the unique homomorphic extension of the atomic interpretation. Then
    \[
    g(s)=\mathcal I^\uparrow(s) \qquad \forall\,s\in \mathsf{Terms}(\Sigma_{\text{atom}},\mathcal G).
    \]
    In the pseudometric case, equality is understood in the metric quotient.
\end{corollary}

\begin{proof}
    By Theorem~\ref{thm:hom}, $g$ is a homomorphism. Zero atomic discrepancy makes it agree with $\mathcal I$ on atoms. The universal property of the typed term algebra gives a unique homomorphic extension of that assignment, so $g=\mathcal I^\uparrow$. The same reasoning applies in the metric quotient.
\end{proof}

Both maps are defined on terms. For an unambiguous grammar, each generated surface form has a unique parse, so both maps descend to those forms. For an ambiguous grammar, $g$ already assigns the same value to terms with the same surface realization. Under the corollary's hypotheses, $\mathcal I^\uparrow=g$, so the intended interpretation also agrees across those parses and descends to surface forms. In the pseudometric case, this agreement is modulo zero distance.

\end{document}